\documentclass{article}
\usepackage{ijcai26}

\usepackage{balance}
\usepackage{eso-pic}

\newcommand{\acceptednotice}{%
Accepted for publication in the Proceedings of the 35th International Joint Conference on Artificial Intelligence (IJCAI 2026)%
}

\usepackage{times}
\usepackage[utf8]{inputenc}
\usepackage[small]{caption}
\usepackage{graphicx}
\usepackage{amsmath}
\usepackage{amsthm}
\usepackage{booktabs}
\usepackage[switch]{lineno}

\usepackage{url}
\usepackage[hidelinks]{hyperref}       
\usepackage{booktabs}       
\usepackage{amsfonts}       
\usepackage{nicefrac}       
\usepackage{microtype}      
\usepackage{xcolor}         
\usepackage{latexsym}
\usepackage{amsthm}
\usepackage{enumitem}
\usepackage{color}
\usepackage{times}
\usepackage{soul}
\usepackage{algorithm}
\usepackage{algpseudocode}
\usepackage[normalem]{ulem} 
\usepackage{subcaption}
\usepackage{multirow}%
\usepackage{mathrsfs}%
\usepackage{dsfont}
\usepackage{textcomp}%
\usepackage{manyfoot}%
\usepackage{tabularx}

\usepackage{tikz}
\usepackage{dblfloatfix} 
\usetikzlibrary{arrows.meta,positioning,fit,calc}

\newtheorem{theorem}{Theorem}

\title{LLMs Uncertainty Quantification via Adaptive Conformal Semantic Entropy}

\author{
Hamed Karimi$^1$
\and
Vaishali Meyappan$^1$\and
Reza Samavi$^{1,2}$\\
\affiliations
$^1$
Toronto Metropolitan University, 
Toronto, Ontario, Canada\\
$^2$Vector Institute, Toronto, Ontario, Canada\\
\emails
\{hamed.karimi, 
vaishali.meyappan, samavi\}@torontomu.ca 
}

\begin{document}

\maketitle

\AddToShipoutPictureFG{%
  \AtPageUpperLeft{%
    \raisebox{-0.25in}{%
      \makebox[\paperwidth][c]{\scriptsize \acceptednotice}%
    }%
  }%
}



\begin{abstract}
LLMs' overconfidence, particularly when hallucinating, poses a significant challenge for the deployment of the models in safety-critical settings and makes a reliable estimation of uncertainty necessary. Existing approaches for uncertainty quantification typically prioritize lexical or probabilistic measures; however, these techniques often ignore the semantic variance of different responses with similar meaning. In this paper, we propose Adaptive Conformal Semantic Entropy (ACSE), a method for estimating prompt-level uncertainty by adaptively measuring semantic dispersion in LLMs outputs. 
Our uncertainty scoring function is based on clustering semantic entropy of multiple diverse responses to the same prompt. The function adaptively adjusts the uncertainty score based on semantic features of each cluster. To ensure statistical reliability of our score, we use conformal calibration to apply a decision rule to accept/abstain the prompts, providing a finite-sample, distribution-free guarantee such that the error rate among the accepted responses remains bounded by a user-specified tolerance.
Our extensive experimental evaluations using different LLMs and datasets, demonstrate that our approach consistently outperforms state-of-the-art uncertainty quantification baselines using discriminative performance, conformal guarantees, and probabilistic calibration indicators. As a highlight, for TriviaQA dataset, AUROC of our approach is $0.88$ compared to $0.65$ produced by the token entropy approach. 
\end{abstract}

\section{Introduction} \label{intro}

Despite the wide range application of Large Language Models (LLMs), they remain uncertain in their predictions and often generate incorrect or misleading outputs with high confidence~\cite{shanahan2024talking}. This overconfidence limits their safe deployment in high-stakes domains such as healthcare~\cite{thirunavukarasu2023large}, law~\cite{teubner2023welcome}, and scientific research~\cite{liu2023summary}. Reliable uncertainty quantification (UQ), defined as the process of estimating the model's confidence in its own predictions, is therefore essential for improving model reliability, safety, and trustworthiness~\cite{chang2024survey}.

\paragraph{Related Work.}
Most existing methods for uncertainty estimation in LLMs rely on token-level signals, such as entropy over next-token distributions~\cite{duan2023shifting,fadeeva2024fact,yarie2024mitigating} or average sequence log-likelihood~\cite{yao2019quality,huang2023look}. 
While computationally efficient, these approaches capture only surface-level lexical variability rather than uncertainty over meaning~\cite{shorinwa2025survey}. 
In autoregressive generation, LLMs often assign high probability to multiple near-synonymous tokens, producing outputs that are probabilistically coherent yet semantically ambiguous or incorrect. As a result, token-level entropy cannot distinguish superficial phrasing variation from genuine semantic ambiguity. In particular, entropy may remain low even when probability mass is spread across responses with mutually inconsistent meanings, creating a misleading impression of confidence~\cite{brown2020language,holtzman2020curious}. 

More recent work has shifted toward semantic-level uncertainty estimation, notably Semantic Entropy (SE)~\cite{kuhn2023semantic}, which aggregates uncertainty by clustering semantically equivalent responses using Natural Language Inference. However, SE relies on hard cluster assignments, forcing each response into a single cluster and ignoring overlapping semantic regions, which can yield unstable uncertainty estimates. To improve reliability, conformalized approaches~\cite{kaur2024addressinguncertaintyllmsenhance} employ dynamic clustering with finite-sample guarantees, while the Conformalized Abstention Policy~\cite{tayebati2025learning} integrates reinforcement learning to learn instance-specific risk thresholds. Despite their guarantees, these methods treat uncertainty as a scalar signal and remain insensitive to cluster brittleness. As a result, they are vulnerable to the wrong-consensus trap, in which lexically consistent but factually incorrect response clusters satisfy conformal thresholds, masking underlying semantic error.

\begin{figure*}[t]
\centering
\includegraphics[width=1.0\textwidth]{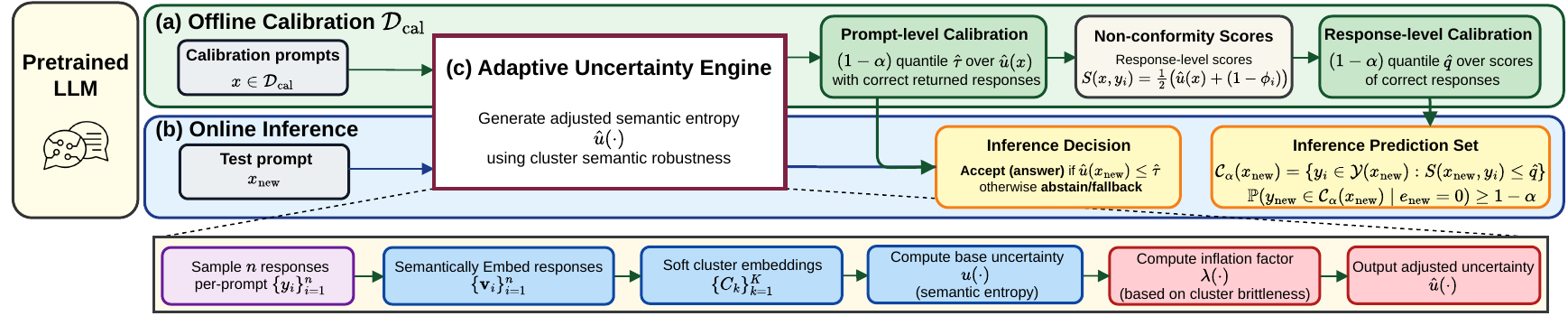}
\caption{ACSE Pipeline. (a) To calibrate a pretrained LLM, for each prompt $x \in \mathcal{D}_{cal}$, we use an Adaptive Uncertainty Engine (AUE) to compute an adjusted uncertainty score, $\hat u(x)$. The $\hat u(x)$ is then used to compute a cutoff for a conformalized decision rule (prompt-level) and  prediction set in response-level for the user defined upper bound error guarantee of $\alpha$. (b) In inference time, for a new prompt, the same AUE is used to compute $\hat u(x)$ to compare with the calibration cutoff. (c) AUE starts with sampling $n$ responses, embeds them, and soft-clusters the embeddings, leading to the computation of base semantic entropy $u(x)$.} 
\label{fig:pipeline}
\end{figure*}

\paragraph{Proposal.} To address these limitations, we propose \textit{Adaptive Conformal Semantic Entropy} (ACSE), a model-independent method that quantifies uncertainty in pretrained LLMs based on dispersion in meaning. Rather than relying on token-level signals, ACSE estimates prompt-level uncertainty in a continuous semantic vector space, capturing dispersion in response meaning. For a finite set of representative prompts (e.g., 1,000 calibration prompts), we sample multiple responses per prompt, embed each response using a sentence encoder whose cosine similarity reflects semantic proximity~\cite{reimers2019sentencebert}, and cluster the embeddings into coherent semantic groups. We then estimate probability mass for each group via soft assignments and derive a base uncertainty score from the normalized entropy of the resulting distribution. This base semantic uncertainty reflects the degree of meaning dispersion: higher dispersion--corresponding to multiple semantically distinct responses--indicates greater uncertainty, whereas semantically consistent responses imply lower uncertainty~\cite{yadkori2024believe}. We subsequently adjust the uncertainty score when cluster structure indicates brittleness, such as weak support for the dominant cluster, high internal diversity, or other risk patterns that empirically correlate with errors under distribution shift. Finally, we calibrate the adjusted uncertainty score using conformal prediction on held-out calibration prompts~\cite{vovk2005algorithmic} to learn a cutoff that serves two purposes: (1) making prompt-level answer/abstain decisions and (2) generating response-level prediction sets to support user decision-making. 
Figure~\ref{fig:pipeline} provides an overview of the approach, which is illustrated by a toy example.

\paragraph{Toy Example.}
Consider a user querying a pretrained geographical LLM with the prompt \emph{What is the capital of Australia?}. The model returns \emph{Sydney} with probability $70\%$ (high confidence) despite its semantic incorrectness, and \emph{Canberra} with probability $30\%$ (low confidence). Selecting the most likely sequence yields a \emph{wrong consensus}, a common failure mode in open-ended generation.
Assume the user has calibrated the model on a held-out set of geographical prompts and identified $0.58$ as a conformity cutoff corresponding to $90\%$ confidence under the assumption that new prompts are drawn from the same distribution~\cite{vovk2005algorithmic}. For a new prompt, ACSE samples $10$ responses using fixed decoding settings: $9$ contain \emph{Sydney} and $1$ contains \emph{Canberra}. These responses are embedded, clustered by meaning, and assigned soft semantic memberships via normalized cosine similarity to cluster centroids. Since assignments are probabilistic, clusters are non-disjoint. Averaging memberships across responses yields mean cluster weights of $0.87$ for \emph{Sydney} and $0.13$ for \emph{Canberra}.
The resulting base semantic uncertainty score, given by the normalized entropy of these weights, is $0.55$. Because this score falls below the abstention threshold ($0.55<0.58$), the model would return the dominant \emph{Sydney} cluster. However, semantic entropy alone can be overly optimistic when clusters are internally brittle, e.g., when cluster membership is unstable or weakly representative. Our approach accounts for such brittleness by adjusting the uncertainty score. In this example, incorporating cluster quality factors (Section~\ref{sec:inflatedScore}) increases the score to $0.62$. Since $0.62>0.58$, the model abstains from returning the incorrect answer. While sampling and post-processing incur additional cost, this overhead is justified in safety-critical applications such as using LLMs for mental health support.

\paragraph{Contributions.}
We make the following contributions. First, we propose ACSE, a method for estimating entropy-based uncertainty at the semantic level for a given prompt. Second, we introduce an adjusted uncertainty score that leverages cluster robustness features to adaptively inflate semantic uncertainty, explicitly penalizing ambiguous response semantics and mitigating hallucinations. Third, we incorporate a post-hoc conformal calibration phase on held-out prompts to learn a cutoff over adjusted uncertainty estimates aligned with response correctness and with formal guarantee, yielding decision rules in both prompt-level and response level. Finally, we demonstrate through extensive experiments on open-domain question answering benchmarks that our adaptive semantic-level uncertainty modeling achieves superior discrimination and calibration compared to state-of-the-art baselines. 

\section{Prompt-Level Uncertainty Estimation}
\label{method}

ACSE calibrates a pretrained LLM to a new domain using a finite set of labeled calibration prompts $\mathcal D_{\mathrm{cal}}$ of size $M$. Calibration begins by sampling multiple responses per prompt to assess the consistency of the model's generated meanings.

\subsection{Response Sample Generation}
Let $\mathcal V$ denote the token vocabulary. A generated response is a finite token sequence $y_i=\langle y_{i,1},\dots,y_{i,T_i}\rangle$ terminated by an EOS token in $T_i$ decoding steps. For a fixed input prompt $x\in\mathcal X$, we draw a set of $n$ independent responses $\mathcal{Y}(x)=\{y_1,\dots,y_n\}$ from a language model with parameters $\theta$.
At each decoding step $t$, the model defines a next-token probability distribution over tokens $v\in\mathcal V$ as,
\begin{equation}
\pi_{i,t}(v)=p_\theta\!\left(y_{i,t}=v \mid x, y_{i,1:t-1}\right), \quad \sum_{v\in\mathcal V}\pi_{i,t}(v)=1\ ,
\end{equation}
where $y_{i,t}$ denotes the $t$-th token of response $y_i$. To balance diversity and semantic coherence, we adopt nucleus (top-$\eta$) sampling~\cite{holtzman2020curious}, which dynamically truncates the distribution at each step. Specifically, the nucleus $\Gamma_\eta(x,i,t)$ is the smallest subset of tokens whose cumulative probability exceeds a threshold $\eta\in\mathbb{R}^{(0,1]}$,
\begin{equation}
\Gamma_\eta(x,i,t)=\operatorname*{arg\,min}_{L\subseteq\mathcal V}
\left\{|L|:\sum_{v\in S}\pi_{i,t}(v)\ge\eta\right\}.
\end{equation}
The truncated distribution is obtained by renormalizing over $\Gamma_\eta$,
\begin{equation}
\pi_{i,t}^{(\eta)}(v)=
\begin{cases}
\displaystyle
\frac{\pi_{i,t}(v)}{\sum_{v\in\Gamma_\eta(x,i,t)}\pi_{i,t}(v)}, & v\in\Gamma_\eta(x,i,t),\\[10pt]
0, & \text{otherwise}.
\end{cases}
\end{equation}
Tokens are sampled as $y_{i,t}\sim\pi_{i,t}^{(\eta)}(\cdot)$ and appended until EOS or a length limit is reached. Repeating this procedure yields the response set $\mathcal Y(x)$.
The variability of $\mathcal Y(x)$ provides an empirical proxy for model confidence. When the model is confident, stochastic decoding produces lexically diverse but semantically equivalent responses; when uncertain, samples diverge into multiple, often conflicting, meanings. The sample size $n$ controls a trade-off between computational cost and resolution: larger $n$ captures finer-grained semantic variation, while smaller $n$ suffices when responses concentrate on a single meaning. In the next step, we quantify this dispersion by embedding $\mathcal Y(x)$ into a semantic space where cosine similarity reflects semantic proximity.

\subsection{Semantic Embedding}
We represent each sampled response by its underlying meaning rather than its surface form, so that paraphrases are embedded nearby while semantically distinct answers are well separated. This representation enables reliable grouping of responses into semantic clusters for downstream uncertainty estimation. 
Each generated response $y_i \in \mathcal{Y}(x)$ is mapped into a continuous semantic space using a pre-trained sentence encoder $f:\mathcal{Y}\rightarrow\mathbb{R}^d$, producing embeddings 
\begin{equation}
\mathcal{E} = \left\{ \mathbf{v}_i = f(y_i) \in \mathbb{R}^d \;\middle|\; y_i \in \mathcal{Y}(x) \right\} .
\end{equation}
We $\ell_2$-normalize these embeddings such that $\|\mathbf{v}_i\|_2=1$ and measure similarity using cosine geometry~\cite{cer2018universal,reimers2019sentencebert,gao2021simcse}. This choice reflects semantic relatedness across lexical and syntactic variation, provides bounded and scale-invariant similarity scores, and yields more stable neighborhoods in high-dimensional spaces than Euclidean distance. 
The resulting embeddings capture the latent semantic structure of model outputs and form the basis for measuring dispersion in meaning. The effectiveness of this representation depends on selecting a sentence encoder that reliably separates semantic similarity from surface variation.
Using the embeddings $\mathcal{E}$, we define a binary error function $e:\mathcal{X}\times\mathcal{Y}\rightarrow\{0,1\}$ that assigns error labels to prompt-response pairs where $e(x,y_i)=\mathds{1}\{y_i \text{ is incorrect}\}$. Labels are determined automatically via semantic matching~\cite{cer2018universal,reimers2019sentencebert}. Given a reference response $\hat y$ for prompt $x$ and a cosine similarity threshold $\tau_{\text{cos}}\in(0,1)$, a response $y_i$ is labeled correct if $\cos(\mathbf{v}_{y_i},\mathbf{v}_{\hat y}) \ge \tau_{\text{cos}}$. In the following, we leverage semantic clusters of $\mathcal{E}$ to derive soft assignments that preserve the semantic ranking induced by the encoder.

\subsection{Soft Clustering of Semantic Embedding}
To transform the unstructured distribution of response embeddings into a meaningful uncertainty signal, we aggregate semantically proximal embeddings into clusters that represent distinct meanings. This serves two purposes: (1) collapsing lexically diverse but semantically equivalent responses into a single equivalence class, and (2) exposing higher-level semantic structure whose relative prevalence reflects model uncertainty.
We apply Hierarchical Agglomerative Clustering (HAC)~\cite{murtagh2017algorithms} to the embedding set $\mathcal E$ under cosine geometry. Using unit-normalized embeddings, cosine dissimilarity reduces to
\begin{equation}
\mathrm{dist}(\mathbf v,\hat{\mathbf v}) = 1-\mathbf v^\top \hat{\mathbf v}\ , \quad \text{s.t.} \quad \|\mathbf v\|_2=\|\hat{\mathbf v}\|_2=1\ .
\end{equation}
HAC constructs a dendrogram by iteratively merging clusters with minimum average pairwise dissimilarity (average linkage), balancing sensitivity to local neighborhoods and global coherence. Cutting the dendrogram at a threshold $\epsilon\in\mathbb{R}^{(0,1)}$ yields a set of $K$ semantically coherent clusters 
\begin{equation}
\mathcal{J}=\{C_1,\dots,C_K\}\ , \qquad \text{where} \qquad \bigcup_{k=1}^K C_k=\mathcal E\ .
\end{equation}
Lower $\epsilon$ enforces strict semantic agreement and higher $\epsilon$ allows broader conceptual grouping. This adaptive thresholding aligns cluster granularity with the intrinsic density of the embedding space.

To avoid overstating certainty near cluster boundaries, we adopt soft cluster assignments. Each cluster $C_k$ is represented by a centroid $\mathbf c_k=\sum_{\mathbf v\in C_k}\mathbf v / |C_k|$.
For a response embedding $\mathbf v_i$, we compute its similarity to each centroid as,
\begin{equation}
a_{ik}=\frac{1}{2}\Big ( 1+\cos(\mathbf v_i,\mathbf c_k)\Big )\ \in \mathbb{R}^{[0,1]}\ ,
\label{centroid_sim}
\end{equation}
which is a parameter-free, monotonic transformation that preserves angular geometry without temperature tuning. Then, normalizing across clusters yields soft assignments
\begin{equation}
s_{ik}=\frac{a_{ik}}{\sum_{k=1}^K a_{ik}}\ , \qquad \text{s.t.} \qquad \sum_{k=1}^K s_{ik}=1\ ,
\end{equation}
where $s_{ik}$ quantifies the degree to which response $y_i$ supports semantic cluster $C_k$. Assignments concentrate when a response aligns strongly with one cluster, and disperse when it lies between competing meanings, preserving signals of semantic ambiguity.
Finally, we aggregate soft assignments across the $n$ sampled responses for prompt $x$ to obtain a prompt-level distribution over meanings $\mathcal P(C_k)={n}^{-1}\sum_{i=1}^n s_{ik}$ where $\sum_{k=1}^K \mathcal P(C_k)=1$.
We define the \emph{Semantic Entropy (SE)} as,
\begin{equation}
\mathcal H_{\mathrm{sem}}(x)=-\sum_{k=1}^K \mathcal P(C_k)\log \mathcal P(C_k)\ ,
\end{equation}
and normalize it to ensure comparability across prompts with different $K$,
\begin{equation}
u(x)=
\begin{cases}
\displaystyle \frac{\mathcal H_{\mathrm{sem}}(x)}{\log K}\ , & K\ge2\ ,\\[6pt]
0\ , & K=1\ .
\end{cases}
\label{eq:cse_normalized}
\end{equation}
The normalized score $u(x)\in\mathbb{R}^{[0,1]}$ quantifies semantic dispersion: low values indicate strong agreement on a single meaning, while high values reflect uncertainty arising from competing semantic interpretations.

\section{Conformalized Adaptive Semantic Entropy} 
\label{conformal_cse}

\paragraph{Conformal Prediction.}
Conformal Prediction (CP) is a distribution-free framework for calibrating set-valued predictions with finite-sample guarantees. Given a calibration dataset $\mathcal D_{\mathrm{cal}}=\{(x_i,y_i)\}_{i=1}^M$ and a test point $(x_{M+1},y_{M+1})$ drawn from the same distribution, CP ensures
\begin{equation}
\mathbb P\!\left(y_{M+1}\in\mathcal C(x_{M+1})\right)\ge 1-\alpha\ ,
\label{}
\end{equation}
where $\alpha\in\mathbb{R}^{(0,1)}$ is the miscoverage level and $\mathcal C$ is the prediction set~\cite{vovk2005algorithmic}. CP relies on a non-conformity score $S:\mathcal X\times\mathcal Y\to\mathbb R^+$ that measures how poorly a candidate output conforms to the calibration data. For confidence level $1-\alpha$, the prediction set is
\begin{equation}
\mathcal C(x)=\{\hat y\in\mathcal Y:\ S(x,\hat y)\le\tau_\alpha\}\ ,
\label{}
\end{equation}
where $\tau_\alpha$ is the empirical $(1-\alpha)$-quantile of calibration scores. This guarantee is model- and distribution-agnostic with exchangeability assumption, i.e., the joint distribution is invariant under permutations of the calibration samples. We refer to~\cite{angelopoulos2023conformal} for a comprehensive treatment.
We adapt CP calibration to LLMs for two complementary purposes: (1) to derive a prompt-level decision rule that determines whether to answer or abstain, and (2) to construct a response-level prediction set over sampled outputs that supports user decision-making by identifying responses satisfying conformal coverage. Achieving these goals hinges on a critical component of CP: the fidelity of the non-conformity score. Since the quality of conformal prediction sets depends directly on this score, we introduce an adjusted uncertainty score tailored to LLMs (Section~\ref{sec:inflatedScore}).


\subsection{Adjusted Semantic Uncertainty}
\label{sec:inflatedScore}
Although the base semantic uncertainty $u(x)$ in~\eqref{eq:cse_normalized} captures dispersion across meanings, it does not account for structural properties of semantic clusters that correlate with prediction difficulty. In particular, $u(x)$ can underestimate uncertainty in low-sample regimes, semantically heterogeneous clusters with weak separation, and sparsely supported clusters that are vulnerable to distribution shift. To address this, we amplify $u(x)$ using the \emph{odds of uncertainty}, $O(u,x)=u(x)/(1-u(x))$, which represents the uncertainty-confidence ratio. To capture cluster brittleness, we define a prompt-level inflation factor $\lambda:\mathcal X\to\mathbb{R}^{[1,\lambda_{\max}]}$ to scale these odds such that $O(\hat u,x)=\lambda(x) \cdot O(u,x)$, leading to the adjusted uncertainty
\begin{equation}
\hat u(x)=\frac{\lambda(x)\,u(x)}{1+\bigl(\lambda(x)-1\bigr)\,u(x)}\ \in\mathbb{R}^{[u(x),\,1]}\ .
\label{eq_inflated_uncer}
\end{equation}
This bounded, monotone, and order-preserving transform ensures $\hat u(x)\ge u(x)$ with equality iff $\lambda(x)=1$ or $u(x)\in\{0,1\}$. Inflation increases with both $u$ and $\lambda$, peaking at intermediate $u(x)$ values to prioritize sensitive borderline decisions. Unlike response-level error labels $e(x,y)$, the inflation factor $\lambda(x)$ operates strictly at the prompt level by aggregating properties of the full response set $\mathcal Y(x)$.

To make $\lambda(x)$ interpretable and data-driven, we synthesize five normalized, prompt-level robustness features from the embedding geometry using $\ell_2$-normalized response embeddings $\{\mathbf{v}_i\}_{i=1}^n$, soft assignments $s_{ik}$ of response $y_i$ to cluster $C_k$, and the agglomerative cluster set $\mathcal{J}$ with unit-norm centroids $\{\mathbf{c}_k\}_{k=1}^K$, where $k^\star=\arg\max_k \mathcal P(C_k)$ denotes the dominant cluster and ${i^\star}=\arg\max_i s_{ik^\star}$ denotes the response most representative of the dominant cluster $C_{k^*}$.

\paragraph{(1) Semantic Entropy.}
The normalized semantic entropy $u(x)$ measures multimodality across clusters and serves as the primary ambiguity signal. Larger $u(x)$ directly increases $\lambda(x)$, inflating uncertainty when probability mass is spread across competing meanings.

\paragraph{(2) Centroid Distance.}
We assess if the dominant response is geometrically supported by its cluster using the centroid similarity $a_{i^\star k^\star}$ 
in~\eqref{centroid_sim}. The distance $\tilde a(x)=1-a_{i^\star k^\star}$ is small when the representative response lies near the cluster center and large when it is atypical, indicating structural instability. We set $\tilde a(x)\propto\lambda(x)$ so that uncertainty increases whenever the dominant prediction lacks sufficient geometric support.

\paragraph{(3) Dominant Cluster Dispersion.}
To quantify internal coherence, we define the intra-cluster dispersion
\begin{equation}
d_{k^\star}(x) = \frac{1}{2|C_{k^\star}|}\sum_{\mathbf v\in C_{k^\star}} \Bigl(1-\cos\!\big(\mathbf v,\,\mathbf c_{k^\star}\big)\Bigr)\ \in\mathbb{R}^{[0,1]}\ .
\label{}
\end{equation}
Low values indicate tight semantic agreement, while high values reflect weak internal consistency; $d_{k^\star}(x)$ increases $\lambda(x)$ proportionally; thus, uncertainty is inflated whenever the dominant cluster lacks internal coherence.

\paragraph{(4) Dominant Cluster Size.}
To penalize fragile consensus supported by few samples, we define the sparsity feature
\begin{equation}
g_{k^\star}(x) = \min\Bigl\{1,\ \frac{\kappa}{|C_{k^\star}|}\Bigr\}\ \in\mathbb{R}^{[0,1]}\ ,
\label{}
\end{equation}
where $\kappa=\mathrm{median}_{x\in\mathcal{D}_{\mathrm{cal}}}\big(\max_k |C_k|\big)$ is a dataset-specific scaling constant computed once on $\mathcal{D}_{\mathrm{cal}}$. Smaller dominant clusters yield larger $g_{k^\star}(x)$ directly proportional to $\lambda(x)$ to inflate uncertainty under sparse semantic support.

\paragraph{(5) Margin to Threshold (Overconfidence Suppression).}
To suppress unwarranted confidence in the low-uncertainty regime, we introduce the margin-to-threshold feature $m(x)$. Using the calibration set $\mathcal D_{\mathrm{cal}}$, we define a label-free reference threshold $\tau_{\mathrm{ref}} = \mathrm{Quantile}_\gamma(\{u(x):x\in\mathcal D_{\mathrm{cal}}\})$ where $\gamma\ge 0.5$ is user specified. This threshold is fixed after calibration and reused during deployment to compute $\lambda(x)$, thereby preserving the validity of subsequent conformal calibration. We define the normalized margin
\begin{equation}
m(x) = \max\Bigl\{0,\ 1-\frac{u(x)}{\tau_{\mathrm{ref}}}\Bigr\}\ \in\mathbb{R}^{[0,1]}\ ,
\label{}
\end{equation}
which is large when $u(x)\ll\tau_{\mathrm{ref}}$, indicating extreme, and potentially miscalibrated confidence. We incorporate $m(x)$ into $\lambda(x)$ so that overconfident prompts receive additional inflation, nudging borderline cases toward abstention, while the $\max$ operator ensures the penalty vanishes whenever $u(x)\ge\tau_{\mathrm{ref}}$.

All features are computed identically during calibration and inference to ensure representational consistency. The constants $\kappa$ and $\tau_{\mathrm{ref}}$ are derived once from the unlabeled calibration set $\mathcal D_{\mathrm{cal}}$ and frozen thereafter, preserving conformal validity. 
To design our monotone inflation function $\lambda$, we introduce a non-negative weight vector $\mathbf{w}=\langle w_u,w_{\tilde a},w_d,w_g,w_m\rangle\in{[0,1]}^5$ that encodes the relative contribution of each structural semantic feature of the feature set $\mathcal{F}=\{u,\tilde a,d_{k^\star},g_{k^\star},m\}$ where $\sum_{l\in\mathcal F} w_l=1$. These assigned weights are a design choice and not learned parameters, accounting for prioritization over cluster brittleness features. We aggregate the feature set $\mathcal{F}$ by weighted average, into the normalized composite brittleness metric $B(x)={\sum_{l\in\mathcal{F}} w_l\,l(x)}\ \in\mathbb{R}^{[0,1]}$ 
where larger $B(x)$ indicates a less reliable prompt due to cluster brittleness. We then define the inflation mapping 
\begin{equation}
\lambda(x)=\frac{2}{\,2-B(x)\,}\ \in\mathbb{R}^{[1,2]}\ ,
\label{eq_lambda_nonlin}
\end{equation}
which satisfies $\lambda_{\mathrm{min}}(x)=1$ for benign prompts ($B(x)=0$) and approaches the bounded maximum $\lambda_{\mathrm{max}}(x)=2$ under maximal structural brittleness, while remaining strictly increasing and convex in $B(x)$. This yields a smooth convex map from \([0,1]\) to \([1,2]\) ensuring boundedness, monotonicity, and order preservation while amplifying high-brittleness cases. Since $\partial\lambda/\partial l>0$ for all $l\in\mathcal F$, any increase in cluster brittleness inflates $\hat u(x)$ via~\eqref{eq_inflated_uncer}. As the entire inflation procedure is label-free, conformal calibration and inference remain valid in the transformed $\hat u$-space, preserving finite-sample guarantees.


\begin{table*}[t]
    \centering
    \resizebox{\textwidth}{!}{
    \begin{tabular}{l|ccccc|ccccc|ccccc|ccccc|ccccc}
        \toprule
        \multirow{2}{*}{\textbf{Method}} & \multicolumn{5}{c|}{\textbf{TriviaQA}} & \multicolumn{5}{c|}{\textbf{CoQA}} & \multicolumn{5}{c|}{\textbf{NQ}} & \multicolumn{5}{c|}{\textbf{TruthfulQA}} & \multicolumn{5}{c}{\textbf{MMLU}} \\
         & \rotatebox{90}{AUROC} & \rotatebox{90}{FPR@95} & \rotatebox{90}{FPR@90} & \rotatebox{90}{AUPR} & \rotatebox{90}{AUARC} & \rotatebox{90}{AUROC} & \rotatebox{90}{FPR@95} & \rotatebox{90}{FPR@90} & \rotatebox{90}{AUPR} & \rotatebox{90}{AUARC} & \rotatebox{90}{AUROC} & \rotatebox{90}{FPR@95} & \rotatebox{90}{FPR@90} & \rotatebox{90}{AUPR} & \rotatebox{90}{AUARC} & \rotatebox{90}{AUROC} & \rotatebox{90}{FPR@95} & \rotatebox{90}{FPR@90} & \rotatebox{90}{AUPR} & \rotatebox{90}{AUARC} & \rotatebox{90}{AUROC} & \rotatebox{90}{FPR@95} & \rotatebox{90}{FPR@90} & \rotatebox{90}{AUPR} & \rotatebox{90}{AUARC} \\
        \midrule
        TE & 0.65 & 0.78 & 0.72 & 0.63 & 0.61 & 0.60 & 0.85 & 0.79 & 0.54 & 0.58 & 0.62 & 0.82 & 0.75 & 0.57 & 0.60 & 0.55 & 0.88 & 0.81 & 0.50 & 0.53 & 0.52 & 0.90 & 0.84 & 0.50 & 0.52 \\
        P(True) & 0.70 & 0.72 & 0.65 & 0.68 & 0.65 & 0.66 & 0.75 & 0.68 & 0.63 & 0.62 & 0.67 & 0.74 & 0.66 & 0.64 & 0.64 & 0.60 & 0.79 & 0.71 & 0.55 & 0.59 & 0.58 & 0.85 & 0.78 & 0.54 & 0.55 \\
        EigV & 0.71 & 0.65 & 0.58 & 0.69 & 0.69 & 0.73 & 0.68 & 0.61 & 0.67 & 0.66 & 0.74 & 0.66 & 0.59 & 0.70 & 0.69 & 0.70 & 0.72 & 0.64 & 0.66 & 0.64 & 0.68 & 0.75 & 0.68 & 0.64 & 0.62 \\
        SU & 0.73 & 0.55 & 0.49 & 0.73 & 0.72 & 0.74 & 0.60 & 0.53 & 0.71 & 0.70 & 0.75 & 0.58 & 0.50 & 0.72 & 0.73 & 0.71 & 0.65 & 0.57 & 0.68 & 0.68 & 0.69 & 0.69 & 0.61 & 0.65 & 0.65 \\
        DDCRP-CP & 0.77 & 0.52 & 0.46 & 0.76 & 0.74 & 0.77 & 0.55 & 0.48 & 0.74 & 0.72 & 0.78 & 0.54 & 0.46 & 0.75 & 0.75 & 0.74 & 0.61 & 0.53 & 0.71 & 0.70 & 0.71 & 0.66 & 0.58 & 0.68 & 0.67 \\
        CAP & 0.80 & 0.48 & 0.41 & 0.79 & 0.76 & 0.79 & 0.52 & 0.44 & 0.77 & 0.75 & 0.80 & 0.50 & 0.42 & 0.78 & 0.77 & 0.76 & 0.58 & 0.50 & 0.73 & 0.72 & 0.73 & 0.62 & 0.54 & 0.70 & 0.69 \\
        \midrule
       \textbf{ACSE (Ours)} & \textbf{0.88} & \textbf{0.31} & \textbf{0.28} & \textbf{0.86} & \textbf{0.85} & \textbf{0.87} & \textbf{0.37} & \textbf{0.32} & \textbf{0.84} & \textbf{0.83} & \textbf{0.84} & \textbf{0.36} & \textbf{0.29} & \textbf{0.81} & \textbf{0.87} & \textbf{0.82} & \textbf{0.43} & \textbf{0.36} & \textbf{0.77} & \textbf{0.76} & \textbf{0.80} & \textbf{0.46} & \textbf{0.39} & \textbf{0.73} & \textbf{0.72} \\
        \bottomrule
    \end{tabular}
    }
    \caption{Hallucination detection performance across diverse benchmarks using Mistral-7b. Metrics evaluate discrimination (AUROC $\uparrow$, AUPR $\uparrow$), safety (FPR@95 $\downarrow$, FPR@90 $\downarrow$), and selective generation (AUARC $\uparrow$).}
        \label{tab:discrimination_dataset}
\end{table*}

\begin{table*}[t]
\centering
\resizebox{\textwidth}{!}{
\begin{tabular}{l|ccccc|ccccc|ccccc|ccccc}
\toprule
\multirow{2}{*}{\textbf{Method}} & \multicolumn{5}{c|}{\textbf{Mistral-7B}} & \multicolumn{5}{c|}{\textbf{LLaMA-2-7B}} & \multicolumn{5}{c|}{\textbf{Falcon-7B}} & \multicolumn{5}{c}{\textbf{Qwen-7B}} \\
 & \rotatebox{90}{AUROC} & \rotatebox{90}{FPR@95} & \rotatebox{90}{FPR@90} & \rotatebox{90}{AUPR} & \rotatebox{90}{AUARC} & \rotatebox{90}{AUROC} & \rotatebox{90}{FPR@95} & \rotatebox{90}{FPR@90} & \rotatebox{90}{AUPR} & \rotatebox{90}{AUARC} & \rotatebox{90}{AUROC} & \rotatebox{90}{FPR@95} & \rotatebox{90}{FPR@90} & \rotatebox{90}{AUPR} & \rotatebox{90}{AUARC} & \rotatebox{90}{AUROC} & \rotatebox{90}{FPR@95} & \rotatebox{90}{FPR@90} & \rotatebox{90}{AUPR} & \rotatebox{90}{AUARC} \\
\midrule
TE & 0.65 & 0.78 & 0.72 & 0.63 & 0.61 & 0.59 & 0.85 & 0.78 & 0.56 & 0.55 & 0.55 & 0.89 & 0.82 & 0.53 & 0.52 & 0.64 & 0.80 & 0.72 & 0.61 & 0.62 \\
P(True) & 0.70 & 0.72 & 0.65 & 0.68 & 0.65 & 0.65 & 0.77 & 0.69 & 0.61 & 0.61 & 0.61 & 0.81 & 0.74 & 0.58 & 0.59 & 0.70 & 0.72 & 0.65 & 0.67 & 0.68 \\
EigV & 0.71 & 0.65 & 0.58 & 0.69 & 0.69 & 0.69 & 0.66 & 0.59 & 0.66 & 0.67 & 0.67 & 0.72 & 0.64 & 0.62 & 0.64 & 0.74 & 0.61 & 0.54 & 0.71 & 0.72 \\
SU & 0.73 & 0.55 & 0.49 & 0.73 & 0.72 & 0.75 & 0.56 & 0.49 & 0.72 & 0.72 & 0.72 & 0.60 & 0.53 & 0.70 & 0.68 & 0.79 & 0.50 & 0.43 & 0.77 & 0.78 \\
DDCRP-CP & 0.77 & 0.52 & 0.46 & 0.76 & 0.74 & 0.77 & 0.52 & 0.45 & 0.74 & 0.74 & 0.74 & 0.57 & 0.49 & 0.71 & 0.72 & 0.82 & 0.47 & 0.40 & 0.79 & 0.81 \\
CAP & 0.80 & 0.48 & 0.41 & 0.79 & 0.76 & 0.79 & 0.48 & 0.41 & 0.76 & 0.76 & 0.76 & 0.54 & 0.46 & 0.73 & 0.74 & 0.84 & 0.43 & 0.36 & 0.81 & 0.83 \\
\midrule
\textbf{ACSE (Ours)} & \textbf{0.89} & \textbf{0.28} & \textbf{0.24} & \textbf{0.88} & \textbf{0.87} & \textbf{0.86} & \textbf{0.35} & \textbf{0.29} & \textbf{0.83} & \textbf{0.84} & \textbf{0.84} & \textbf{0.38} & \textbf{0.32} & \textbf{0.81} & \textbf{0.82} & \textbf{0.91} & \textbf{0.25} & \textbf{0.19} & \textbf{0.89} & \textbf{0.90} \\
\bottomrule
\end{tabular} }
\caption{Hallucination detection performance on TriviaQA across diverse architectures. Metrics evaluate discrimination (AUROC $\uparrow$, AUPR $\uparrow$), safety (FPR@95 $\downarrow$, FPR@90 $\downarrow$), and selective generation (AUARC $\uparrow$).}
\label{tab:discrimination_model}
\end{table*}


\subsection{Prompt Acceptance and Response Set}
\label{prompt-response-cp}

\paragraph{Prompt-Level Decision Rule.}
For each prompt $x$, let define the returned response $\tilde y=y_{i^*}$ as the most representative member of dominant cluster $k^\star$, i.e., cluster with largest total assignment mass. 
We use the inflated uncertainty $\hat u(x)\in\mathbb{R}^{[0,1]}$ as a prompt-level score so that small $\hat u(x)$ indicates robust semantic agreement after inflation, while large $\hat u(x)$ indicates cluster brittleness. Since CP requires one label per prompt, we define a prompt-level error label by evaluating the returned response $E(x)=e(x,\tilde y)\in\{0,1\}$ where $E(x)=0$ indicates that the returned response is correct.
Using $\mathcal{I}_0=\{\hat u(x): x\in\mathcal{D}_{\mathrm{cal}} \land E(x)=0\}$ as calibration prompts whose returned response is correct, we calibrate a prompt acceptance threshold $\hat\tau = \mathrm{Quantile}_{1-\alpha}\big(\mathcal{I}_0\big)$
which is $(1-\alpha)$-quantile of scores $\hat u(.)$ in the multiset $\mathcal I_0$ (if $\mathcal I_0=\emptyset$, we abstain on all prompts).
At inference, to answer a prompt $x$, we
\begin{equation}
\text{accept }\ x\ \Longleftrightarrow\ \hat u(x)\le \hat\tau\ ;\quad \text{abstain otherwise.}
\label{eq:twolevel_accept_prompt}
\end{equation}
This calibration is appropriate since it aligns the prompt-level score $\hat u(x)$ with the prompt-level LLM outcome, i.e., correctness of the actually returned response $\tilde y$.

\begin{theorem}[Conformal coverage for prompt acceptance]
\label{thm:twolevel_prompt}
Assume $\tilde y$ and $\hat u(\cdot)$ are computed by the same procedure for calibration and test prompts under exchangeability of $\mathcal{D}_{\mathrm{cal}}$ and test prompt $x_{\mathrm{new}}$. With the $(1-\alpha)$-quantile $\hat\tau$ defined on $\hat u$-space of prompts with correct returned responses, we have 
\begin{align}
\begin{split}
&\mathbb{P}\!\left(\hat u(x_{\mathrm{new}})\le \hat\tau\ \middle|\ E(x_{\mathrm{new}})=0\right) \ge 1-\alpha\ .
\end{split}
\label{eq:twolevel_prompt_guar}
\end{align}
\end{theorem}


\begin{table*}[t]
\centering
\resizebox{\textwidth}{!}{
\begin{tabular}{l|cccccc|cccccc|cccccc}
\toprule
\multirow{3}{*}{\textbf{Method}} &
\multicolumn{6}{c|}{$\alpha=0.05$ (Strict)} &
\multicolumn{6}{c|}{$\alpha=0.10$ (Standard)} &
\multicolumn{6}{c}{$\alpha=0.20$ (Permissive)} \\
& \multicolumn{3}{c}{Response-level} & \multicolumn{3}{c|}{Prompt-level}
& \multicolumn{3}{c}{Response-level} & \multicolumn{3}{c|}{Prompt-level}
& \multicolumn{3}{c}{Response-level} & \multicolumn{3}{c}{Prompt-level} \\
& R-Cov $\uparrow$ & APS $\downarrow$ & SSCV $\downarrow$ & P-Cov $\uparrow$ & Acc. $\uparrow$ & Risk $\downarrow$
& R-Cov $\uparrow$ & APS $\downarrow$ & SSCV $\downarrow$ & P-Cov $\uparrow$ & Acc. $\uparrow$ & Risk $\downarrow$
& R-Cov $\uparrow$ & APS $\downarrow$ & SSCV $\downarrow$ & P-Cov $\uparrow$ & Acc. $\uparrow$ & Risk $\downarrow$ \\
\cmidrule(lr){1-1}\cmidrule(lr){2-4}\cmidrule(lr){5-7}\cmidrule(lr){8-10}\cmidrule(lr){11-13}\cmidrule(lr){14-16}\cmidrule(lr){17-19}
SU       & 0.823 & \textbf{1.05} & 0.081 & 0.810 & 34.2\% & 0.064 & 0.816 & \textbf{1.02} & 0.089 & 0.820 & 52.1\% & 0.118 & 0.808 & \textbf{1.01} & 0.115 & 0.852 & 68.4\% & 0.225 \\
DDCRP-CP & 0.941 & 1.95 & 0.046 & 0.942 & 42.5\% & \textbf{0.038} & 0.923 & 1.45 & 0.051 & 0.884 & 62.8\% & \textbf{0.076} & 0.810 & 1.22 & 0.074 & 0.791 & 79.1\% & \textbf{0.152} \\
CAP      & 0.953 & 1.72 & 0.038 & 0.951 & 46.1\% & 0.046 & 0.902 & 1.32 & 0.042 & 0.917 & 65.4\% & 0.093 & 0.827 & 1.15 & 0.062 & 0.814 & 82.5\% & 0.188 \\
\midrule
\textbf{ACSE (Ours)} & \textbf{0.985} & 1.32 & \textbf{0.024} & \textbf{0.971} & \textbf{55.4\%} & 0.044 & \textbf{0.939} & 1.07 & \textbf{0.030} & \textbf{0.923} & \textbf{75.8\%} & 0.088 & \textbf{0.842} & 1.08 & \textbf{0.045} & \textbf{0.835} & \textbf{89.4\%} & 0.178 \\
\bottomrule
\end{tabular}
}
\caption{Conformal analysis across varying $\alpha$: Response Cov., APS, SSCV, Prompt Cov., Acceptance Rate (Acc.), and Selective Risk.}
\label{tab:coverage_risk_guarantee}
\end{table*}

\paragraph{Response-Level Prediction Sets.}
To construct CP-consistent prediction sets over sampled responses, we define a response-level non-conformity score that combines prompt uncertainty with response atypicality. Using the clustering outputs, we first define a response conformity measure
\begin{equation}
\phi_i\;=\; s_{i\hat{k}}\cdot\mathcal{P}(C_{\hat k})\ \in\mathbb{R}^{[0,1]} \quad \text{s.t.} \quad \hat k=\arg\max_{k} s_{ik}\ ,
\label{eq:twolevel_q}
\end{equation}
which is large when response $y_i$ is both strongly assigned to its best-matching semantic cluster and that cluster is well supported across samples.
We then define the response-level non-conformity score (smaller is better) as,
\begin{equation}
S(x,y_i)\;=\;\frac{1}{2}\Big(\hat u(x) + \big(1-\phi_i\big)\Big)\ \in\ \mathbb{R}^{[0,1]}\ ,
\label{eq:twolevel_S}
\end{equation}
which increases with prompt-level uncertainty $\hat u(x)$ and with response atypicality. Thus, a response attains a low score only when it arises from a low-uncertainty prompt and is semantically representative within a well-supported cluster.
For calibration prompt-response pairs, we compute $S(x_j,y_{ji})$ in~\eqref{eq:twolevel_S} and form the multiset $\mathcal{S}_0=\{S(x_j,y_{ji}):\ e_{ji}=0\}$ of scores corresponding to correct responses. The $(1-\alpha)$-quantile threshold is then defined as $\hat q=\mathrm{Quantile}_{1-\alpha}(\mathcal S_0)$. For a test prompt $x$ with sampled responses $\mathcal Y(x)=\{y_i\}_{i=1}^n$, the response-level prediction set is
\begin{equation}
\mathcal C_\alpha(x)\;=\;\big\{\,y_i\in\mathcal{Y}(x):\ S(x,y_i)\le \hat q\,\big\}\ ,
\label{eq:twolevel_C}
\end{equation}
where the size of $\mathcal C_\alpha(x)$ provides a direct measure of response-level uncertainty. If a single output is required, we return the most representative response $\tilde y\in\mathcal C_\alpha(x)$. 

\begin{theorem}[Conformal coverage for prediction response sets]
\label{thm:twolevel_response}
Assume $S(x,y)$ is computed identically on calibration and test triples $(x,y,e)$ under exchangeability with the same response sampling procedure. Considering $\mathcal C_\alpha$ defined in~\eqref{eq:twolevel_C}, for a randomly drawn test triple $(x_{\mathrm{new}},y_{\mathrm{new}},e_{\mathrm{new}})$,
\begin{equation}
\mathbb{P}\!\left(y_{\mathrm{new}}\in\mathcal C_\alpha(x_{\mathrm{new}})\ \middle|\ e_{\mathrm{new}}=0\right)\ \ge\ 1-\alpha\ .
\label{eq:twolevel_response_guar}
\end{equation}
\end{theorem}
\noindent The proofs of the theorems and computational overhead analysis of our approach are reported in Appendices~\ref{proofs} and~\ref{app:comp}, respectively.

\section{Experimental Evaluations} 
\label{exp}

We evaluate the performance of ACSE using discriminative and conformal analyses, benchmarking its ability to identify hallucinations (incorrect responses) reliably through adaptive uncertainty inflation against both uncalibrated and calibrated baselines. 
Appendix~\ref{exp_setup} describes the experimental setup and evaluation metrics. 

\paragraph{Implementation Details.}
For each prompt, we generate $n=10$ responses via nucleus sampling ($\eta=0.9$) with temperature $0.3$ to balance diversity and coherence and embed them using the sentence encoder all-MiniLM-L6. We perform HAC clustering with average linkage and cosine distance threshold $\epsilon=0.35$. We set $\lambda$ to equally weigh all brittleness features as $\mathbf{w}=\langle\frac15,\frac15,\frac15,\frac15,\frac15\rangle$ to avoid any prioritization, and set the overconfidence penalty threshold $\tau_{\mathrm{ref}}$ to the $75$th percentile of the calibration base-uncertainty distribution. Each dataset uses $2000$ samples split into disjoint calibration set (60\%) and test set (40\%). For Size-Stratified Coverage Violation (SSCV) evaluation, we define prediction set size strata as $[1\text{-}2, 3\text{-}5, 6\text{-}7, 8\text{-}10]$. 
The source codes are available at {\color{blue}\emph{\url{https://github.com/tailabTMU/ACSE}}}.

\paragraph{Datasets.}\label{datasets}
We evaluate on five benchmarks spanning open-domain retrieval, conversational QA, hallucination detection, and multi-domain reasoning: \textit{TriviaQA}~\cite{joshi-etal-2017-triviaqa}, 
\textit{CoQA}~\cite{reddy2019coqa}, \textit{Natural Questions (NQ)}~\cite{kwiatkowski2019natural}, \textit{TruthfulQA}~\cite{lin2022truthfulqa}, and a stratified subset of \textit{MMLU}~\cite{wang2024mmlu}. 

\paragraph{LLM Models.}\label{models}
We use \textit{Mistral-7B-Instruct}~\cite{jiang2024mistral} which combines Grouped-Query Attention (GQA) and Sliding Window Attention (SWA) for strong performance with efficient inference. To assess model-agnosticism, we additionally evaluate \textit{Llama-2-7B-Chat}~\cite{touvron2023llama}, \textit{Falcon-7B-Instruct} (Multi-Query Attention; MQA), and \textit{Qwen-7B-Chat}~\cite{bai2023qwentechnicalreport}.

\paragraph{Baselines.}\label{baseline}
We compare ACSE against six uncertainty methods spanning lexical, semantic, geometric, and calibrated approaches: \textit{Token Entropy (TE)}~\cite{vaswani2017attention}, \textit{P(True)}~\cite{kadavath2022language}, \textit{Semantic Uncertainty (SU)}~\cite{kuhn2023semantic}, and \textit{EigV}~\cite{lin2024generatingconfidenceuncertaintyquantification}. We also include two CP-based calibrated frameworks: \textit{DDCRP-CP}~\cite{kaur2024addressinguncertaintyllmsenhance} and \textit{Conformal Abstention Policy (CAP)}~\cite{tayebati2025learning}; we emphasize comparisons among DDCRP-CP, CAP, and ACSE since they share a CP calibration protocol, enabling standardized evaluation.

\subsection{Results and Interpretations}
\label{results}
We conduct an extensive empirical evaluation of ACSE to assess whether it (1) improves hallucination detection relative to state-of-the-art baselines, (2) strictly enforces user-specified coverage via conformal mechanisms along with the selective risk while maintaining high acceptance rates, and (3) improves uncertainty quantification on hallucinations against baselines. 

\begin{figure}[!t]
    \centering
    \begin{subfigure}[t]{0.45\columnwidth}
        \centering
        \includegraphics[width=\linewidth]{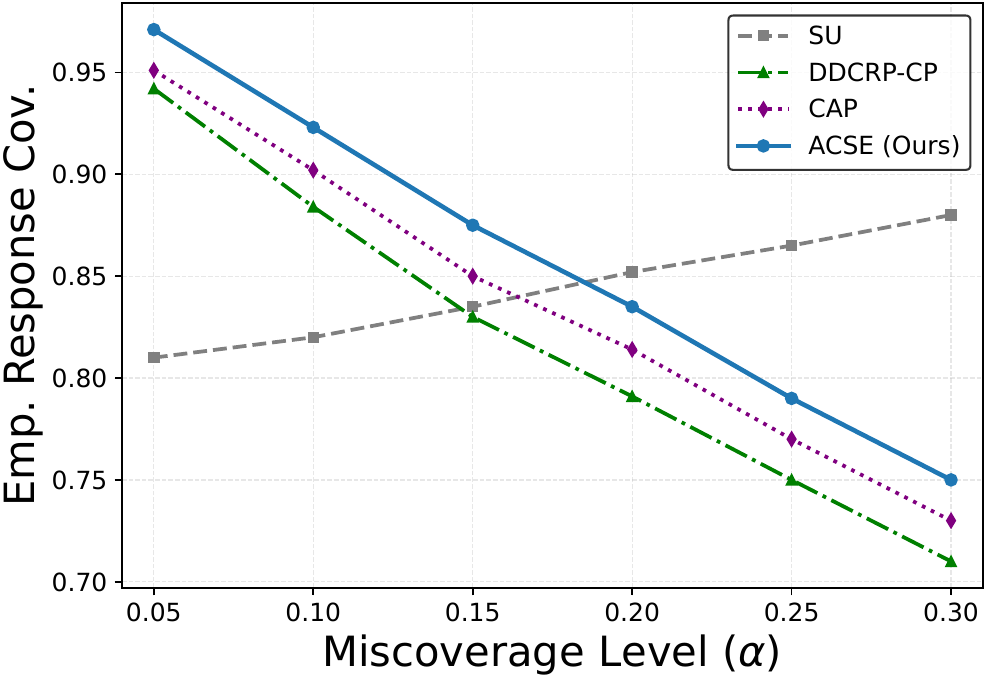}
        \caption{Response Coverage}
        \label{fig:emp_coverage}
    \end{subfigure}\hspace{4mm}
    \begin{subfigure}[t]{0.45\columnwidth}
        \centering
        \includegraphics[width=\linewidth]{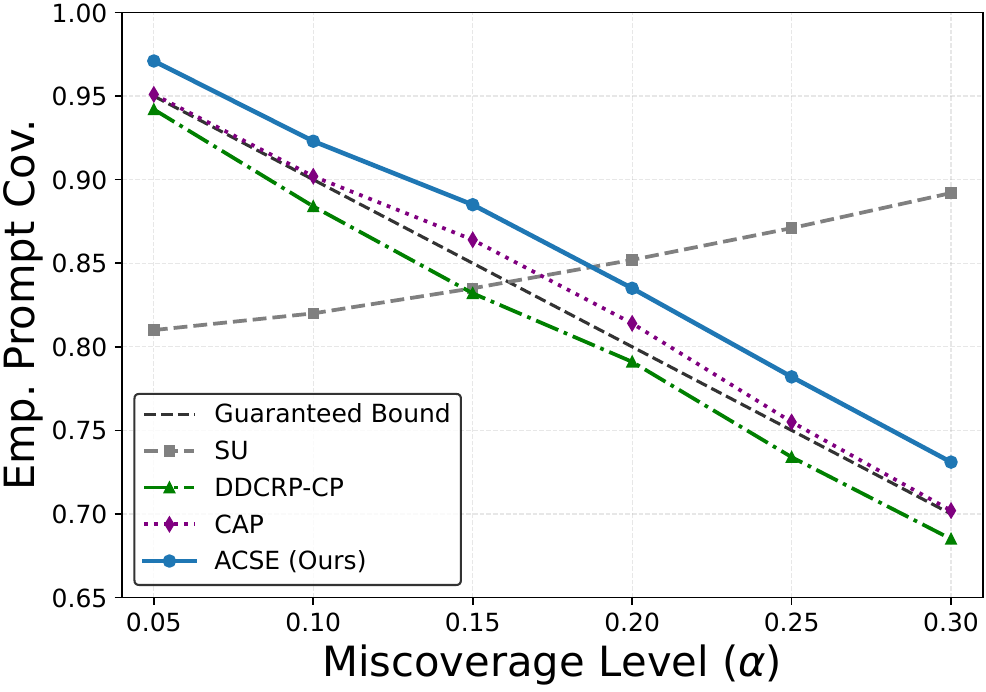}
        \caption{Prompt Coverage}
        \label{fig:prompt_cov}
    \end{subfigure}
    \begin{subfigure}[t]{0.45\columnwidth}
        \centering
        \includegraphics[width=\linewidth]{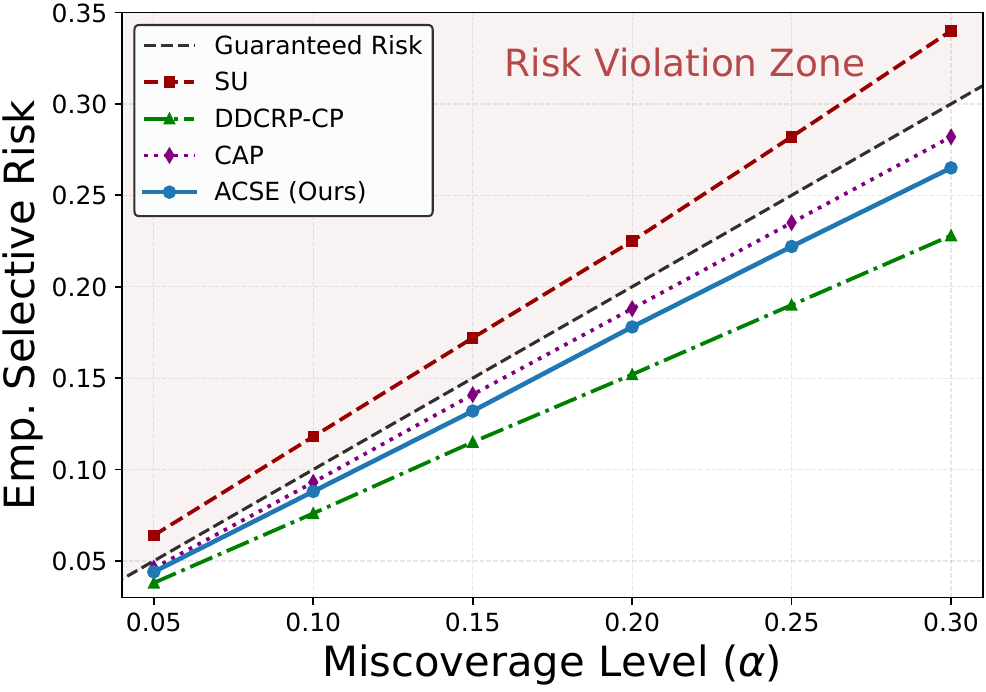}
        \caption{Selective Risk}
        \label{fig:selective_risk_control}
    \end{subfigure}\hspace{4mm}
    \begin{subfigure}[t]{0.45\columnwidth}
        \centering
        \includegraphics[width=\linewidth]{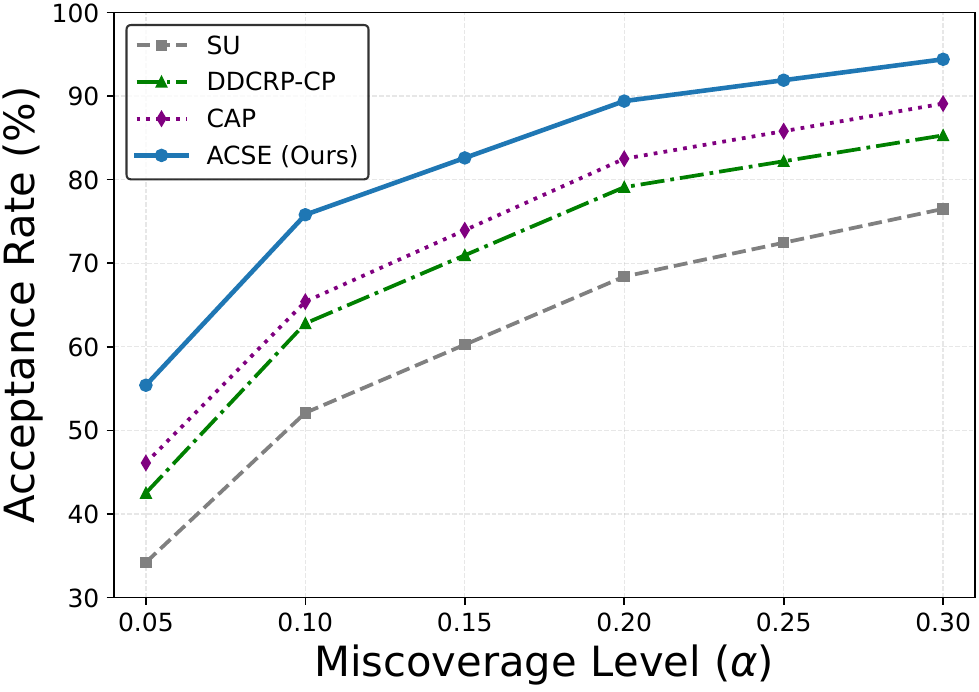}
        \caption{Acceptance Rate}
        \label{fig:coverage_sensitivity_extended}
    \end{subfigure}
    \caption{Sensitivity analysis against varying miscoverage level $\alpha$.}
    \label{fig:sensitivity_alpha}
\end{figure}

\paragraph{Hallucination Detection Performance.}
\label{hallucination}
This evaluation assesses the ACSE's capability to distinguish between factually correct response generations and hallucinations across various datasets and LLMs. The results in Tables~\ref{tab:discrimination_dataset} and~\ref{tab:discrimination_model} demonstrate that ACSE achieves better discriminative performance across all benchmarks, e.g., on TriviaQA, ACSE increases AUROC to $0.88$ from CAP's $0.80$. This improvement is particularly notable on high-efficiency models such as Falcon-7B where ACSE reduces $0.48$ FPR@95 observed by CAP to $0.31$, a $35.4$\% relative decrease in accepted hallucinations. Low standard deviations ($0.01$ to $0.02$) across five independent runs confirm the statistical significance of the results. We report ablation studies and additional results on a text summarization dataset, recent 8B LLMs and baselines, together with base SE comparisons, probabilistic calibration results, and statistical significance tests in Appendix~\ref{additional_res}. 

\paragraph{Conformal Guarantees and Selective Risk.}
\label{risk}
We assess ACSE reliability by varying $\alpha$ using both response-level and prompt-level metrics reported in Table~\ref{tab:coverage_risk_guarantee}, to enforce user-specified guarantees through conformal decision rules against baselines. 
Statistical reliability is reported through Response-level Coverage (R-Cov) which measures the inclusion rate of correct generated responses in the prediction sets, and Prompt-level Coverage (P-Cov) which evaluates the fraction of accepted prompts with correct returned responses. The analyses in Figure~\ref{fig:emp_coverage} show that ACSE guarantees superior R-Cov across all $\alpha$ levels. 
Similarly, Figure~\ref{fig:prompt_cov} demonstrates that ACSE consistently satisfies the theoretical prompt coverage guarantee, whereas baselines frequently exhibit undercoverage violations. ACSE also maintains low Average Prediction Set Size (APS) of $1.07$ compared to $1.32$ APS of CAP at $\alpha=0.10$. 

Additionally, ACSE achieves the lowest SSCV scores across all settings, e.g., $0.030$ vs. $0.042$ for CAP at $\alpha=0.10$, indicating superior calibration stability across varying prediction set sizes. 
ACSE also consistently maximizes the Acceptance Rate (Acc.) across benchmarks, reaching $75.8$\% at $\alpha=0.10$, a substantial improvement over the runner-up CAP at $65.4$\%. 
Furthermore, ACSE strictly guarantees the selective risk as the conditional error among accepted prompts that hallucinate, which is restricted by $\alpha$ as shown in Figure~\ref{fig:selective_risk_control}. While the SU baseline suffers from overconfidence and risk violations, ACSE maintains a selective risk of $0.178$ at low APS. ACSE also avoids the conservativeness of DDCRP-CP, which yields marginally lower risk but produces significantly larger prediction sets ($1.45$ vs. $1.07$), trading precision for safety.

\paragraph{Uncertainty Quantification Analysis.}
This experiment evaluates the discriminative ability of our adjusted uncertainty score ($\hat u$) to separate correct responses from hallucinations. Figure~\ref{fig:scatter_comparison} illustrates this separation across four quadrants against each baseline. The top-left quadrant denotes the ideal abstention region for hallucinations where they exhibit both low baseline confidence and high ACSE uncertainty; conversely, the bottom-right quadrant denotes the ideal acceptance region for the correct responses. ACSE consistently assigns high uncertainty to hallucinations and low uncertainty to correct responses whereas the baselines exhibit varying ranges of confidence while hallucinating.
The top-right quadrant highlights baseline failures compared to ACSE, in which hallucinations are incorrectly assigned high confidence but are correctly flagged by ACSE with high uncertainty. 

\begin{figure}[!t]
    \centering
    \begin{subfigure}[t]{0.36\columnwidth}
        \centering
        \includegraphics[width=\linewidth]{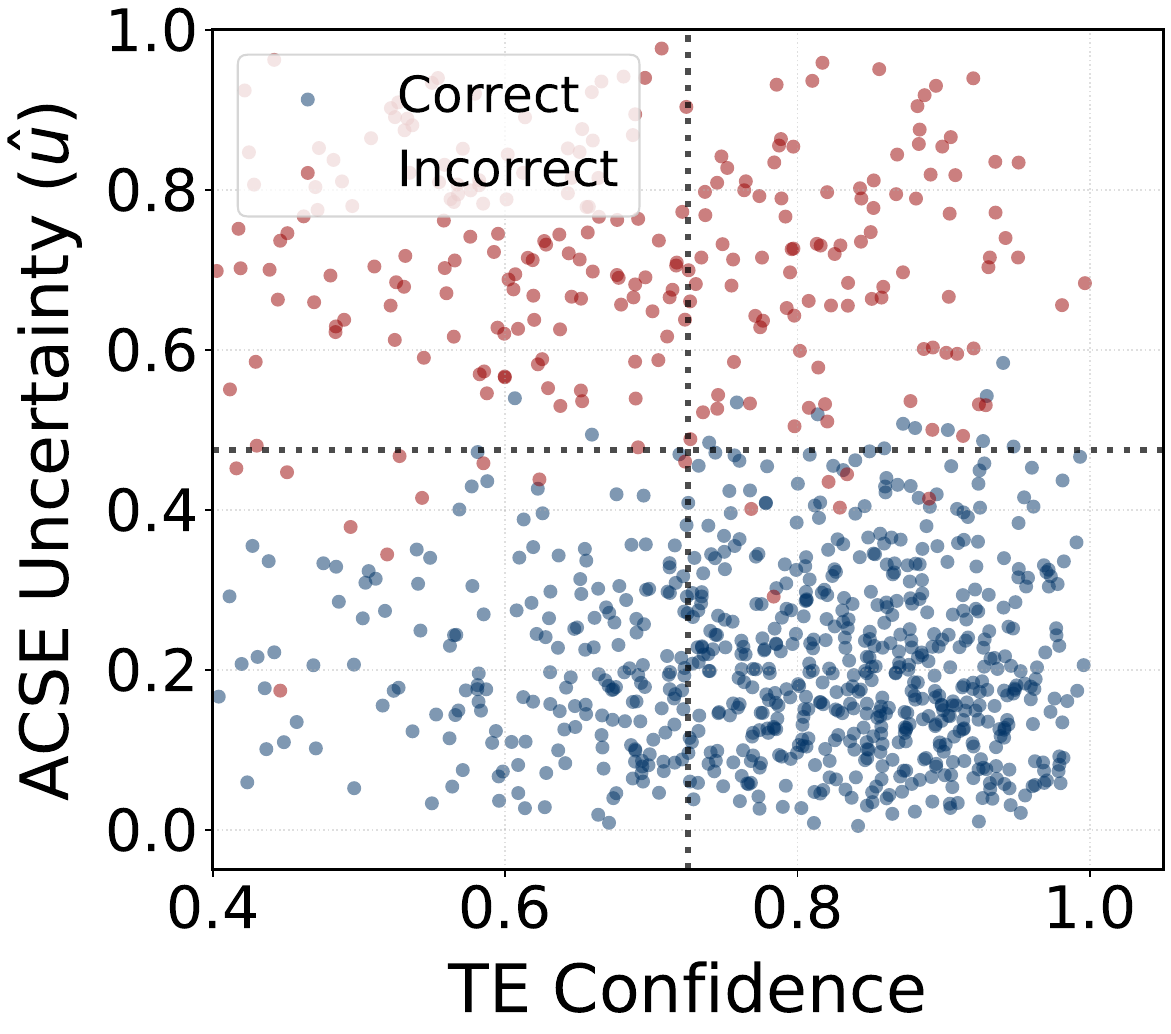}
        \caption{TE}
    \end{subfigure}
    \begin{subfigure}[t]{0.32\columnwidth}
        \centering
        \includegraphics[width=\linewidth]{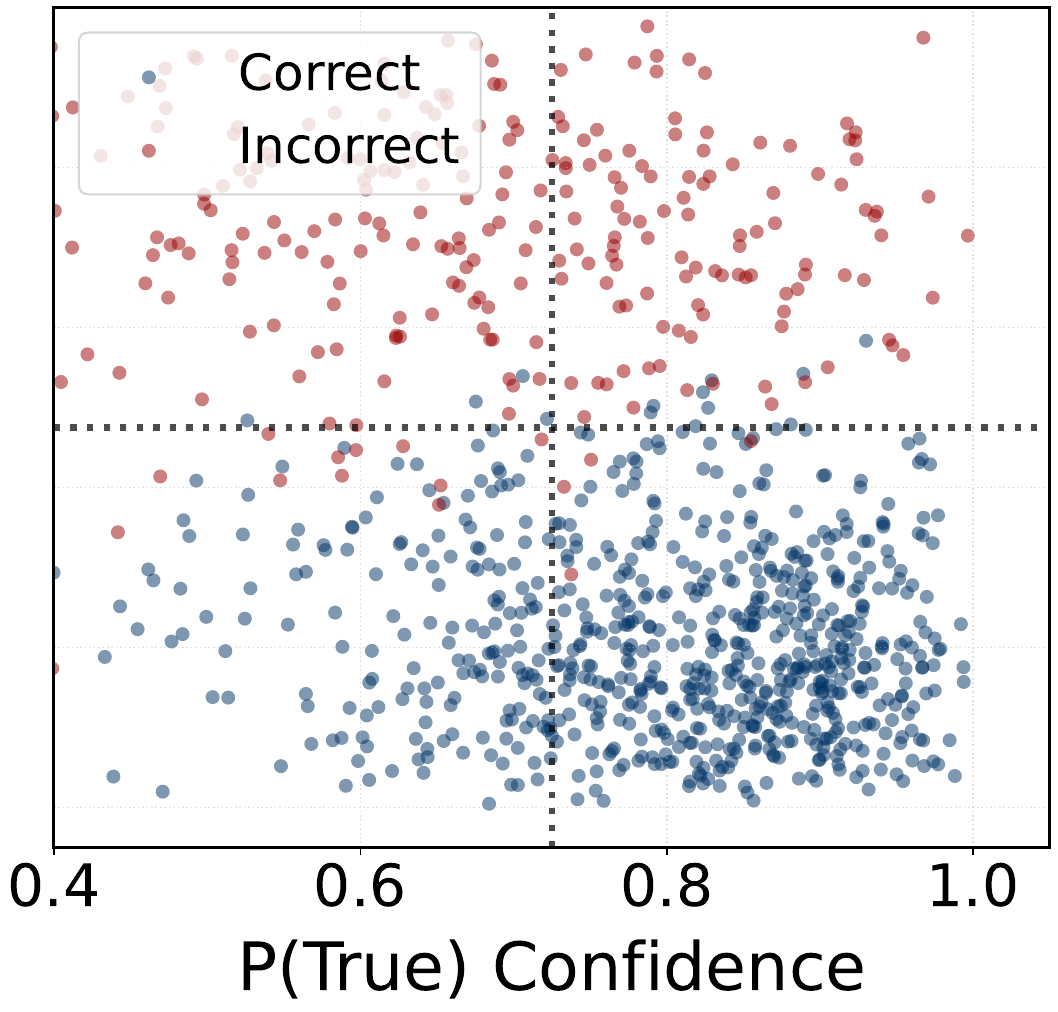}
        \caption{P(True)}
    \end{subfigure}
    \begin{subfigure}[t]{0.32\columnwidth}
        \centering
        \includegraphics[width=\linewidth]{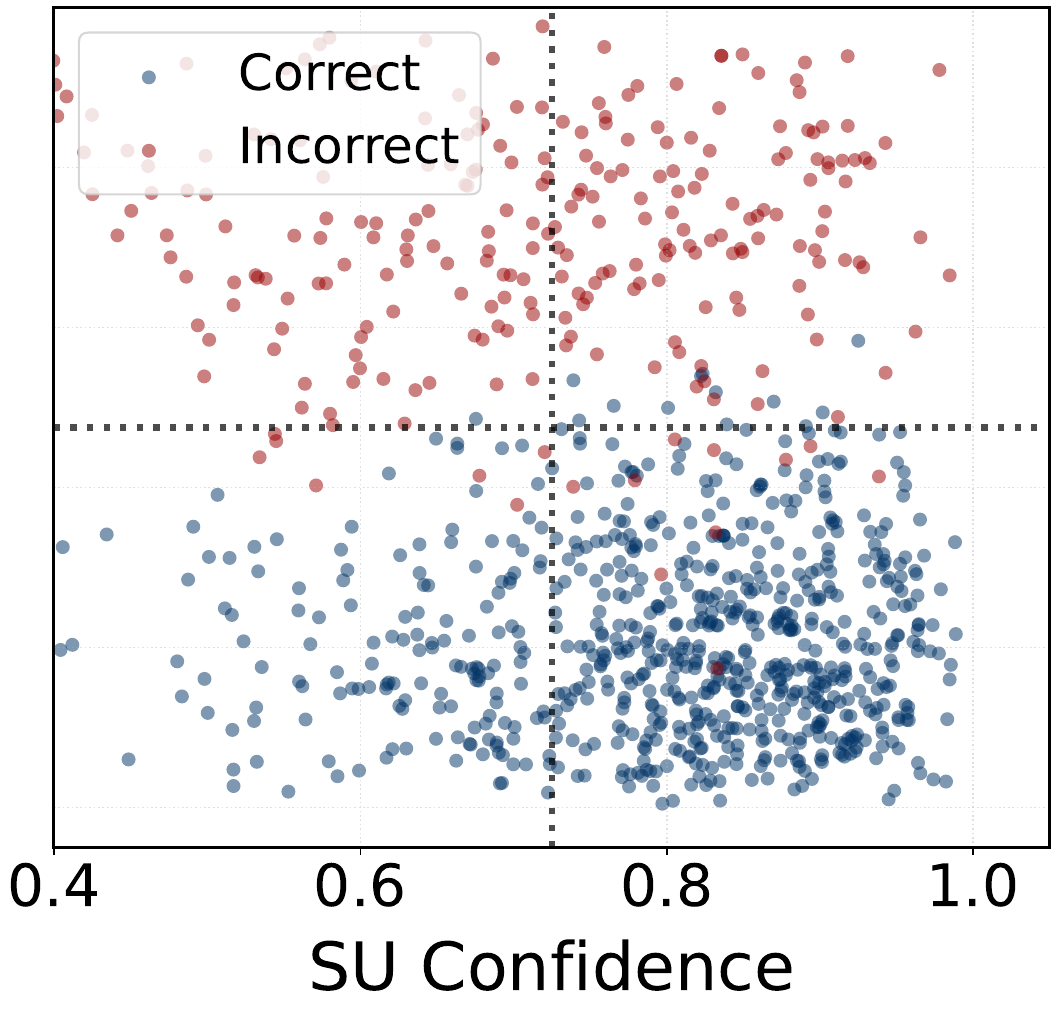}
        \caption{SU}
    \end{subfigure}

    \begin{subfigure}[t]{0.36\columnwidth}
        \centering
        \includegraphics[width=\linewidth]{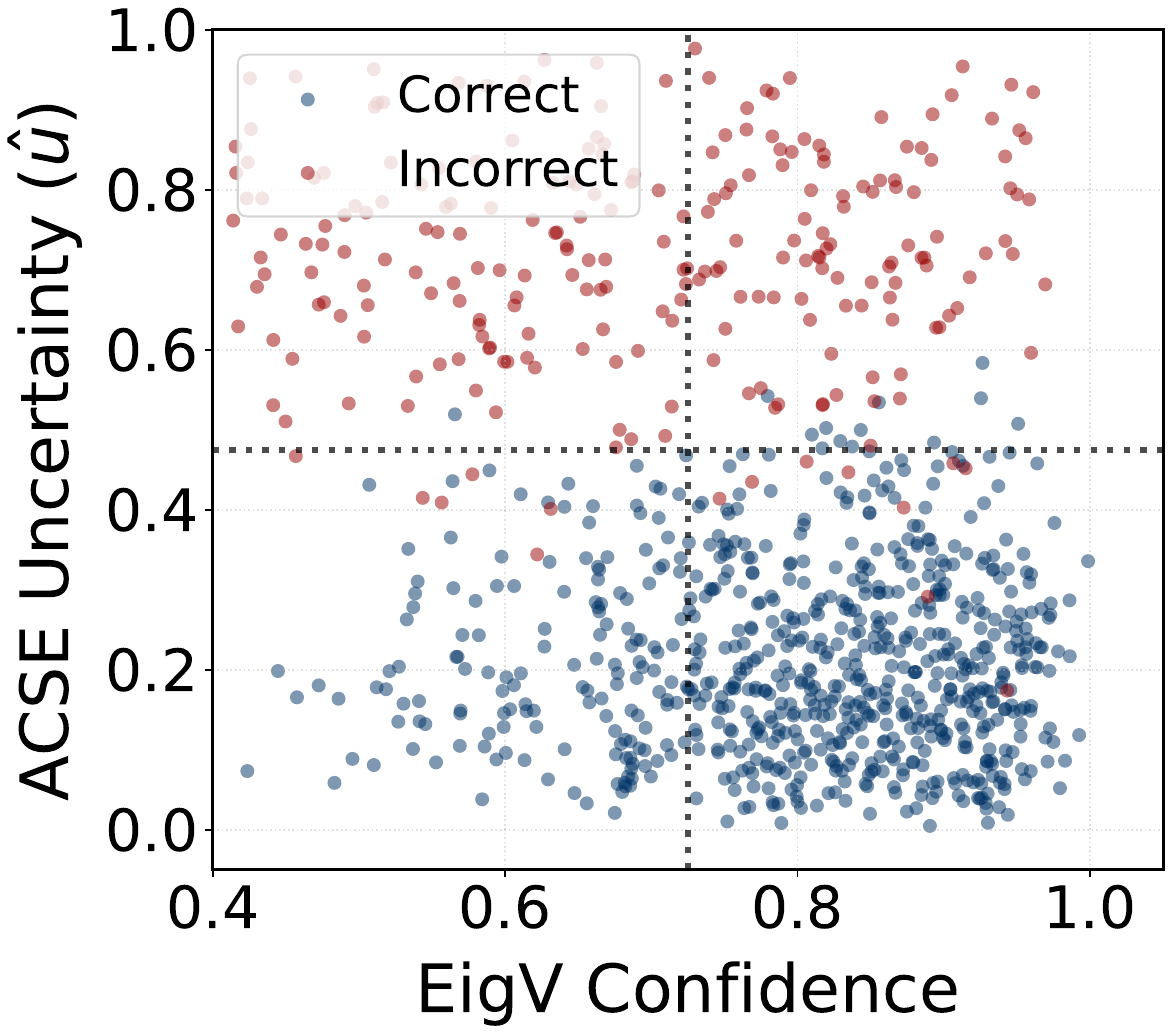}
        \caption{EigV}
    \end{subfigure}
    \begin{subfigure}[t]{0.32\columnwidth}
        \centering
        \includegraphics[width=\linewidth]{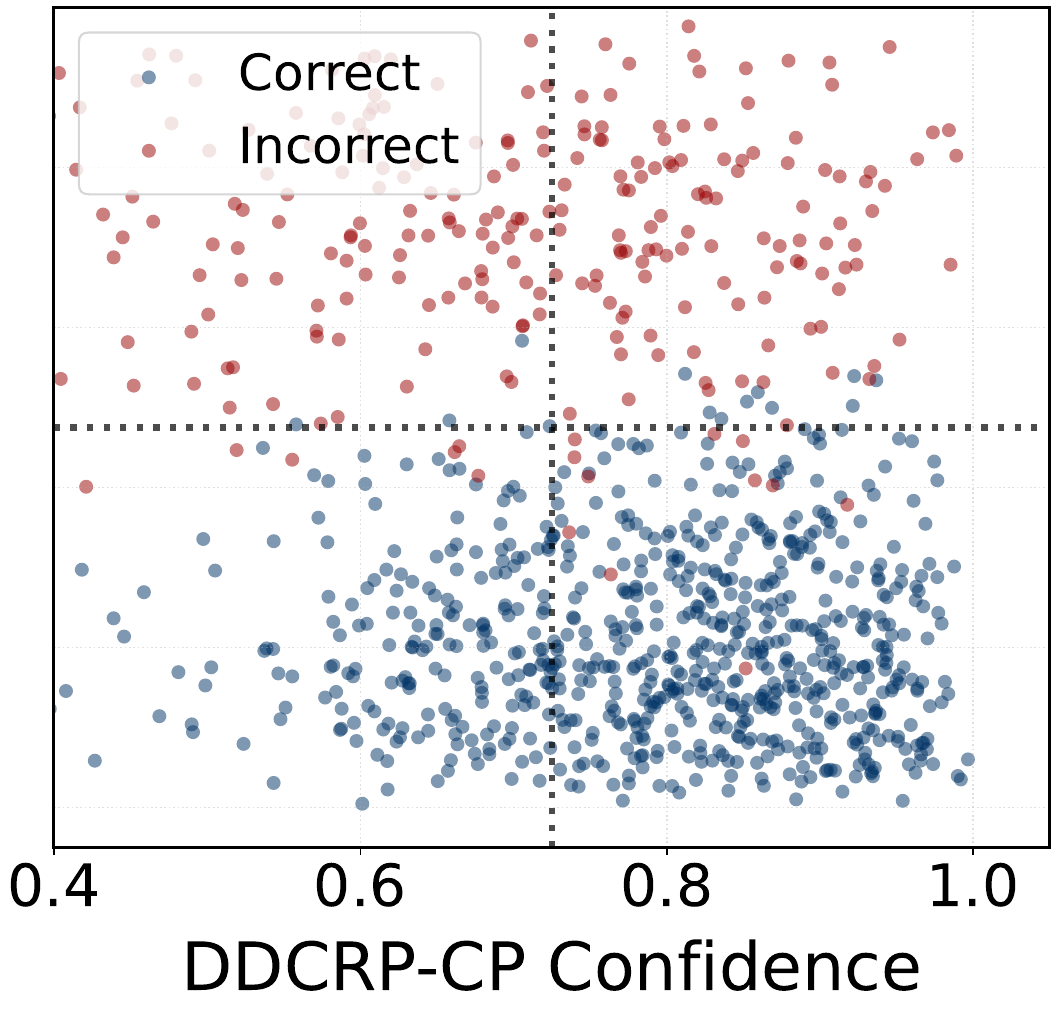}
        \caption{DDCRP-CP}
    \end{subfigure}
    \begin{subfigure}[t]{0.32\columnwidth}
        \centering
        \includegraphics[width=\linewidth]{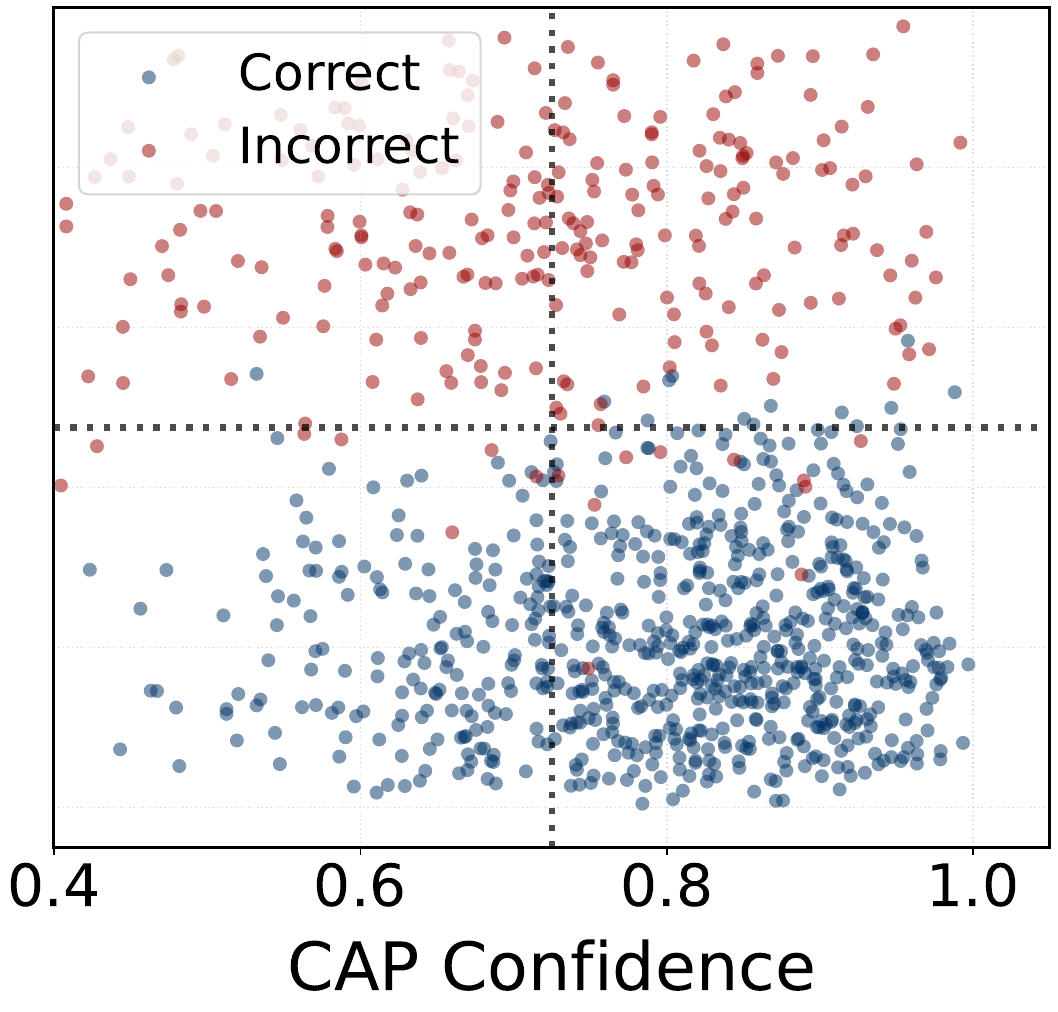}
        \caption{CAP}
    \end{subfigure}

    \caption{Comparing ACSE uncertainty against baseline confidences, distinguishing correct responses (blue) from hallucinations (red).}
    \label{fig:scatter_comparison}
\end{figure}

\section{Conclusion}
\label{conc}
We introduced ACSE to mitigate LLM overconfidence by quantifying uncertainty through semantic dispersion. Rather than treating uncertainty as a one-dimensional signal, ACSE employs an adaptive inflation mechanism to penalize semantic brittleness, ensuring reliable estimates. Evaluations across diverse benchmarks demonstrate that ACSE significantly improves hallucination detection, discriminative performance, and coverage guarantees. 
ACSE offers a robust solution for the reliable deployment of LLMs in safety-critical domains. Future work will explore adaptive sampling strategies to reduce computational overhead and extend the framework to conformal risk control for fine-grained quality calibration.




\section*{Acknowledgments}
This research was undertaken thanks in part to funding from the Canada First Research Excellence Fund at Toronto Metropolitan University and Natural Sciences and Engineering Research Council of Canada (NSERC) Discovery grants (\#348100).



\bibliographystyle{named}
\bibliography{references}

\clearpage
\newpage
\appendix


\section{Proofs}
\label{proofs}

\begin{proof}[Proof of Theorem 1] 
We analyze the probability of acceptance conditional on the event that the returned response is correct, $E(x_{\mathrm{new}})=0$. Let $\mathcal{I}_0 = \{ \hat{u}(x) : x \in \mathcal{D}_{\mathrm{cal}} \land E(x) = 0 \}$ be the multiset of inflated uncertainty scores for the calibration prompts that resulted in correct responses, and let $M_0 = |\mathcal{I}_0|$.
By the assumption of exchangeability between the calibration set $\mathcal{D}_{\mathrm{cal}}$ and the test prompt $x_{\mathrm{new}}$, the score $\hat{u}(x_{\mathrm{new}})$ is exchangeable with the elements of $\mathcal{I}_0$ conditional on $E(x_{\mathrm{new}})=0$. Consequently, the rank of $\hat{u}(x_{\mathrm{new}})$ among the $M_0 + 1$ scores in $\mathcal{I}_0 \cup \{\hat{u}(x_{\mathrm{new}})\}$ is uniformly distributed on $\{1, \dots, M_0+1\}$.
The threshold $\hat{\tau}$ is defined as the empirical $(1-\alpha)$-quantile of $\mathcal{I}_0$. Following standard conformal prediction theory, this quantile is chosen such that it covers at least a fraction $1-\alpha$ of the exchangeable distribution. Therefore,
\begin{equation}
    \mathbb{P}\!\left(\hat{u}(x_{\mathrm{new}}) \le \hat{\tau} \mid E(x_{\mathrm{new}}) = 0\right) \ge 1 - \alpha\ .
\end{equation}
\end{proof}

\begin{proof}[Proof of Theorem 2] 
We analyze the probability of inclusion conditional on the event that the specific response is correct, $e_{\mathrm{new}}=0$. Let $\mathcal{S}_0 = \{ S(x_j, y_{ji}) : e_{ji} = 0 \}$ be the multiset of non-conformity scores for all correct responses in the calibration set.
Due to the exchangeability of the calibration and test data triples $(x,y,e)$, the test score $S(x_{\mathrm{new}}, y_{\mathrm{new}})$ is exchangeable with the elements of $\mathcal{S}_0$ conditional on $e_{\mathrm{new}}=0$. The threshold $\hat{q}$ is defined as the $(1-\alpha)$-quantile of $\mathcal{S}_0$. By the standard conformal prediction guarantee, the probability that a new exchangeable score does not exceed this quantile is at least $1-\alpha$.
Since the prediction set is defined as $\mathcal{C}_\alpha(x_{\mathrm{new}}) = \{ y \in \mathcal{Y}(x_{\mathrm{new}}) : S(x_{\mathrm{new}}, y) \le \hat{q} \}$, the condition $S(x_{\mathrm{new}}, y_{\mathrm{new}}) \le \hat{q}$ is equivalent to $y_{\mathrm{new}} \in \mathcal{C}_\alpha(x_{\mathrm{new}})$. Therefore,
\begin{equation}
    \mathbb{P}\!\left(y_{\mathrm{new}}\in\mathcal C_\alpha(x_{\mathrm{new}})\ \middle|\ e_{\mathrm{new}}=0\right) \ge 1-\alpha\ .
\end{equation}
\end{proof}


\section{Computational Overhead Analysis}
\label{app:comp}

The dominant overhead comes from generating multiple responses per prompt. By contrast, the ACSE-specific semantic module is lightweight and depends only on the small bounded sampling budget \(n\) (in our experiments, \(n=10\)). This additional cost is justified in safety-critical settings, where improved reliability and calibrated abstention are more important than maximizing raw throughput.

For a prompt \(x\), ACSE samples \(n\) responses \(y_1,\dots,y_n\). Let \(C_{\mathrm{gen}}(y_i)\) denote the cost of generating response \(y_i\). We keep this term abstract, since its exact value depends on the base LLM, hardware, batching, and implementation. Let \(d\) be the embedding dimension and \(K\) the number of semantic clusters returned by HAC, where \(1< K\le n\). The per-prompt runtime can be written as
\begin{equation}
\mathcal T(x)=\sum_{i=1}^{n} C_{\mathrm{gen}}(y_i)+\mathcal O\big(n^2 d + n^2\log n + nKd\big).
\label{eq:acse_runtime_simplified}
\end{equation}
The three semantic terms in~\eqref{eq:acse_runtime_simplified} correspond to (1) \(\mathcal O(n^2 d)\) for pairwise cosine geometry over the \(n\) response embeddings in \(\mathbb R^d\), (2) \(\mathcal O(n^2\log n)\) for average-linkage hierarchical agglomerative clustering, and (3) \(\mathcal O(nKd)\) for centroid-based soft assignment and prompt-level semantic aggregation.
Since \(K\le n\), the semantic post-processing depends only on the small response budget \(n\) and the embedding dimension \(d\), not on the parameter count of the base LLM. Therefore, ACSE is not purely linear in \(n\), but its practical bottleneck remains repeated response generation, not the semantic uncertainty computation itself. In particular, the inflation rule and final prompt-level feature computation are negligible once the responses have already been generated.

The same pipeline is used during calibration and deployment, but its impact differs across phases. During calibration, the cost is paid once offline to compute the conformal operating point. During deployment, the per-prompt latency is governed directly by~\eqref{eq:acse_runtime_simplified}. Thus, the main computational concern is online inference rather than offline calibration.

The ACSE-specific memory overhead comes from storing the semantic objects associated with the \(n\) sampled responses of one prompt: the response embeddings in \(\mathbb R^d\), the pairwise similarity or distance matrix used by HAC, and the response-to-cluster soft assignments. Thus, the per-prompt memory overhead is
\begin{equation}
\mathcal M(x)=\mathcal O(nd+n^2+nK)=\mathcal O(nd+n^2),
\label{eq:acse_memory_simplified}
\end{equation}
since \(K\le n\). Importantly, ACSE does not replicate the base LLM in memory; it only adds a small semantic buffer on top of standard model inference memory. For the bounded regime used in this work (\(n=10\)), this overhead remains modest.

We measure the computational overhead of ACSE, DDCRP-CP, and CAP under the same hardware and decoding setup on TriviaQA using Llama-3-8B. In addition to end-to-end latency, we report a decomposition into shared sampling cost and method-specific post-processing cost, allowing us to separate the cost of uncertainty estimation from the cost of generating multiple responses, which is common to all multi-sample methods.
As shown in Table~\ref{tab:overhead_llama}, ACSE has the highest end-to-end latency among the compared methods, which is expected for a multi-sample semantic uncertainty approach. However, the dominant cost for all three conformal methods comes from shared response sampling rather than from method-specific processing. Relative to one-shot decoding, DDCRP-CP, CAP, and ACSE require $5.7\times$, $6.0\times$, and $6.4\times$ total latency, respectively. The additional ACSE-specific post-processing overhead is limited, exceeding DDCRP-CP by only $0.58$\,s and CAP by $0.31$\,s. Peak memory follows the same trend, with ACSE remaining only slightly above the other multi-sample baselines. These results show that although ACSE is more expensive than simpler conformal methods, most of the overhead comes from shared multi-sample generation, while the extra cost of the adaptive semantic inflation step remains modest.

\begin{table}[t]
    \centering
    \resizebox{0.5\textwidth}{!}{
    \begin{tabular}{l|ccccc}
        \toprule
        \textbf{Method} & Total Latency (s) & Relative Cost & Sampling Time (s) & Post-processing (s) & Peak Memory (GB) \\
        \midrule
        One-shot decoding & $0.86$ & $1.0\times$ & $0.86$ & $0.00$ & $9.3$ \\
        DDCRP-CP & $4.89$ & $5.7\times$ & $4.21$ & $0.68$ & $11.6$ \\
        CAP & $5.16$ & $6.0\times$ & $4.21$ & $0.95$ & $12.0$ \\
        ACSE (Ours) & $5.47$ & $6.4\times$ & $4.21$ & $1.26$ & $12.3$ \\
        \bottomrule
    \end{tabular}
    }
	    \caption{Computational overhead (Llama-3-8B; TriviaQA).}
    \label{tab:overhead_llama}
\end{table}


\begin{table*}[t]
    \centering
    \resizebox{\textwidth}{!}{
    \begin{tabular}{l|ccccc|ccccc}
        \toprule
        \multirow{2}{*}{\textbf{Method}} & \multicolumn{5}{c|}{\textbf{TriviaQA}} & \multicolumn{5}{c}{\textbf{SamSum}} \\
         & {AUROC} & {FPR@95} & {FPR@90} & {AUPR} & {AUARC} & {AUROC} & {FPR@95} & {FPR@90} & {AUPR} & {AUARC} \\
        \midrule
        SAR   & $.836\pm.005$ & $.403\pm.011$ & $.349\pm.010$ & $.821\pm.006$ & $.811\pm.006$ & $.786\pm.006$ & $.492\pm.012$ & $.428\pm.011$ & $.767\pm.007$ & $.756\pm.007$ \\
        KLE   & $.851\pm.004$ & $.374\pm.010$ & $.323\pm.009$ & $.834\pm.005$ & $.823\pm.005$ & $.798\pm.005$ & $.468\pm.011$ & $.408\pm.010$ & $.779\pm.006$ & $.768\pm.006$ \\
        SIMBA & $.844\pm.005$ & $.389\pm.010$ & $.336\pm.009$ & $.828\pm.006$ & $.817\pm.006$ & $.792\pm.006$ & $.481\pm.011$ & $.418\pm.010$ & $.774\pm.006$ & $.762\pm.006$ \\
        \midrule
       \textbf{ACSE (Ours)} & $\mathbf{.872\pm.004}$ & $\mathbf{.332\pm.009}$ & $\mathbf{.286\pm.008}$ & $\mathbf{.849\pm.005}$ & $\mathbf{.844\pm.005}$ & $\mathbf{.807\pm.005}$ & $\mathbf{.447\pm.010}$ & $\mathbf{.389\pm.009}$ & $\mathbf{.786\pm.006}$ & $\mathbf{.776\pm.006}$ \\
        \bottomrule
    \end{tabular}
    }
	    \caption{Task-agnostic discriminative performance on Llama-3-8B.}
    \label{tab:recent_baselines_llama3}
\end{table*}

\begin{table*}[t]
    \centering
    \renewcommand{\arraystretch}{1.20}
    \resizebox{\textwidth}{!}{
    \begin{tabular}{l|ccccc|ccccc|ccccc}
        \toprule
        \multirow{2}{*}{\textbf{Method}} 
        & \multicolumn{5}{c|}{\textbf{Llama-3-8B}} 
        & \multicolumn{5}{c|}{\textbf{Qwen-3-8B}} 
        & \multicolumn{5}{c}{\textbf{Granite-3.1-8B}} \\
        & {AUROC} & {FPR@95} & {FPR@90} & {AUPR} & {AUARC}
        & {AUROC} & {FPR@95} & {FPR@90} & {AUPR} & {AUARC}
        & {AUROC} & {FPR@95} & {FPR@90} & {AUPR} & {AUARC} \\
        \midrule
        SAR
        & $.836\pm.005$ & $.403\pm.011$ & $.349\pm.010$ & $.821\pm.006$ & $.811\pm.006$
        & $.833\pm.005$ & $.410\pm.011$ & $.355\pm.010$ & $.818\pm.006$ & $.808\pm.006$
        & $.826\pm.006$ & $.421\pm.012$ & $.365\pm.011$ & $.810\pm.007$ & $.800\pm.007$ \\
        KLE
        & $.851\pm.004$ & $.374\pm.010$ & $.323\pm.009$ & $.834\pm.005$ & $.823\pm.005$
        & $.848\pm.004$ & $.381\pm.010$ & $.329\pm.009$ & $.830\pm.005$ & $.819\pm.005$
        & $.841\pm.005$ & $.392\pm.011$ & $.339\pm.010$ & $.823\pm.006$ & $.812\pm.006$ \\
        SIMBA
        & $.844\pm.005$ & $.389\pm.010$ & $.336\pm.009$ & $.828\pm.006$ & $.817\pm.006$
        & $.840\pm.005$ & $.397\pm.010$ & $.343\pm.009$ & $.824\pm.006$ & $.813\pm.006$
        & $.834\pm.006$ & $.407\pm.011$ & $.351\pm.010$ & $.817\pm.006$ & $.806\pm.006$ \\
        \midrule
        \textbf{ACSE (Ours)}
        & $\mathbf{.872\pm.004}$ & $\mathbf{.332\pm.009}$ & $\mathbf{.286\pm.008}$ & $\mathbf{.849\pm.005}$ & $\mathbf{.844\pm.005}$
        & $\mathbf{.869\pm.004}$ & $\mathbf{.338\pm.009}$ & $\mathbf{.291\pm.008}$ & $\mathbf{.846\pm.005}$ & $\mathbf{.842\pm.005}$
        & $\mathbf{.863\pm.005}$ & $\mathbf{.349\pm.010}$ & $\mathbf{.301\pm.009}$ & $\mathbf{.839\pm.005}$ & $\mathbf{.836\pm.005}$ \\
        \bottomrule
    \end{tabular}
    }
	    \caption{Discriminative performance across additional models/methods on TriviaQA.}
    \label{tab:model_generalization_rebuttal}
\end{table*}

\begin{table}[t]
\centering
\resizebox{\columnwidth}{!}{%
\begin{tabular}{l|cc|cc|cc}
\toprule
\multirow{2}{*}{\textbf{Method}} & \multicolumn{2}{c|}{\textbf{TriviaQA}} & \multicolumn{2}{c|}{\textbf{CoQA}} & \multicolumn{2}{c}{\textbf{NQ}} \\
 & AUROC & FPR@95 & AUROC & FPR@95 & AUROC & FPR@95 \\
\midrule
SE Score & 0.77 & 0.49 & 0.75 & 0.55 & 0.76 & 0.52 \\
\textbf{ACSE (Inflated)} & \textbf{0.88} & \textbf{0.31} & \textbf{0.87} & \textbf{0.37} & \textbf{0.84} & \textbf{0.36} \\
\midrule
\textit{Absolute Gain} & \textit{+0.11} & \textit{-0.18} & \textit{+0.12} & \textit{-0.18} & \textit{+0.08} & \textit{-0.16} \\
\bottomrule
\end{tabular}
}
\caption{Comparative analysis of ACSE against Base SE score.}
\label{tab:ablation_inflation}
\end{table}

\begin{table}[t]
\centering
\resizebox{\columnwidth}{!}{%
\begin{tabular}{l|cc|cc|cc}
\toprule
\multirow{2}{*}{\textbf{Method}} & \multicolumn{2}{c|}{$\alpha=0.05$ (Strict)} & \multicolumn{2}{c|}{$\alpha=0.10$ (Standard)} & \multicolumn{2}{c}{$\alpha=0.20$ (Permissive)} \\
 & ECE $\downarrow$ & Brier $\downarrow$ & ECE $\downarrow$ & Brier $\downarrow$ & ECE $\downarrow$ & Brier $\downarrow$ \\
\midrule
SU & 0.17 & 0.25 & 0.16 & 0.24 & 0.14 & 0.22 \\
DDCRP-CP & 0.10 & 0.15 & 0.09 & 0.14 & 0.07 & 0.12 \\
CAP & 0.07 & 0.12 & 0.06 & 0.11 & 0.05 & 0.10 \\
\midrule
\textbf{ACSE (Ours)} & \textbf{0.06} & \textbf{0.10} & \textbf{0.04} & \textbf{0.09} & \textbf{0.03} & \textbf{0.08} \\
\bottomrule
\end{tabular}
}
\caption{Probabilistic calibration analysis across varying $\alpha$.}
\label{tab:probabilistic_calibration}
\end{table}

\begin{table}[t]
    \centering
    \resizebox{0.5\textwidth}{!}{
    \begin{tabular}{l|ccc|ccc|ccc}
        \toprule
        \multirow{2}{*}{\textbf{Dataset}}
        & \multicolumn{3}{c|}{\textbf{ACSE vs.\ KLE}}
        & \multicolumn{3}{c|}{\textbf{ACSE vs.\ SAR}}
        & \multicolumn{3}{c}{\textbf{ACSE vs.\ SIMBA}} \\
        & AUROC & FPR@95 & AUARC
        & AUROC & FPR@95 & AUARC
        & AUROC & FPR@95 & AUARC \\
        \midrule
        TriviaQA & 0.004 & 0.002 & 0.006 & 0.001 & 0.001 & 0.002 & 0.002 & 0.001 & 0.004 \\
        SamSum   & 0.018 & 0.015 & 0.021 & 0.006 & 0.004 & 0.009 & 0.011 & 0.009 & 0.016 \\
        \bottomrule
    \end{tabular}
    }
	    \caption{Paired significance tests ($p$-values) for Table~\ref{tab:recent_baselines_llama3}.}
    \label{tab:recent_baselines_significance}
\end{table}


\section{Experimental Setup and Metrics}
\label{exp_setup}
We implement all methods in PyTorch v2.1.2 and HuggingFace Transformers v4.40.0, using sentence-transformers to embed generated responses. 
For valid comparison, the uncalibrated SU baseline is post-hoc calibrated with isotonic regression to map raw scores to observed error frequencies. We use Llama-3-8B as our primary model on TriviaQA as our primary dataset by default. All experiments were run on an AMD Ryzen 9 5900X (12-core) workstation with a single NVIDIA RTX 3060 GPU (16GB VRAM).

\subsection{Evaluation Metrics}
\label{metrics}
Let $\{(x,y)_i \mid x \in \mathcal{D}_{\mathrm{val}}\ ,\ y\in\mathcal{Y}(x)\}_{i=1}^N$ be the set of validation prompt-response pairs, $\hat u(x)$ be an uncertainty score, and accept (answer) each prompt when $\hat u(x)\le\hat\tau$. We compute the \emph{acceptance rate} as, 
\begin{equation}
\mathrm{AccRate}(\hat\tau)
=\frac{1}{N}\sum_{i=1}^N \mathds{1}{\{\hat u(x_i)\le\hat\tau\}}\ .
\label{}
\end{equation}
To avoid ambiguity, we distinguish \emph{conformal set coverage} (membership of the correct response in a prediction set) from \emph{acceptance rate}, i.e., the fraction of prompts on which ACSE returns an answer rather than abstaining.

The empirical selective risk is also defined as conditional error among accepted prompts, indicating the fraction of accepted prompts whose returned responses are incorrect (hallucinating accepted prompts) with binary error labels $\{E_i=e(x,\tilde y)\}_{i=1}^N$ of returned responses $\tilde y$, and computed as,
\begin{equation}
\hat{R}(\hat\tau)=\frac{\sum_{i=1}^N E_i\,\mathds{1}{\{\hat u(x_i)\le\hat\tau\}}}
{\sum_{i=1}^N \mathds{1}{\{\hat u(x_i)\le\hat\tau\}}}\ .
\label{}
\end{equation}
For computing AUARC (Area under the Accuracy-Rejection curve), we calculate the accuracy among accepted prompts as,
\begin{equation}
\mathrm{Acc}(\hat\tau)=\frac{\sum_{i=1}^N (1-E_i)\mathds{1}{\{\hat u(x_i)\le\hat\tau\}}}
{\sum_{i=1}^N \mathds{1}{\{\hat u(x_i)\le\hat\tau\}}}\ ,
\label{}
\end{equation}
which measures the rate of the accepted prompts that have correct returned response.
The rejection rate is computed as,
\begin{equation}
\mathrm{Rej}(\hat\tau)=\frac{1}{N}\sum_{i=1}^N \mathds{1}{\{\hat u(x_i)>\hat\tau\}}\ .
\label{}
\end{equation}
For prediction set $\mathcal{C}_\alpha(\cdot)$, the empirical CP coverage (prediction response set coverage) is computed as, 
\begin{align}
\mathrm{R\text{-}Cov}_{1-\alpha}
&\;=\;
\frac{\sum\limits_{i=1}^{N}
\mathds{1}\!\left\{e_{i}=0\right\}\,
\mathds{1}\!\left\{y_{i}\in \mathcal C_\alpha(x_i)\right\}}
{\sum\limits_{i=1}^{N}\mathds{1}\!\left\{e_{i}=0\right\}} \\
&\;=\;
\frac{1}{N^R_0}\sum_{(j,i):\,e_{ji}=0}\mathds{1}\!\left\{y_{ji}\in \mathcal C_\alpha(x_j)\right\}\ ,
\label{eq:empirical_conformal_coverage}
\end{align}
where $N^R_0=|\{(j,i) \in \mathcal{D}_{\mathrm{test}}:\,e_{ji}=0\}|$ as the size of prompt-response pairs in the test set whose responses are correct.
The prompt-level coverage computes the fraction of test prompts with correct returned responses that are successfully accepted by the calibrated threshold, and is defined as, 
\begin{align}
\mathrm{P\text{-}Cov}_{1-\alpha}
&\;=\;
\frac{\sum\limits_{i=1}^{N}
\mathds{1}\!\left\{E_{i}=0\right\}\,
\mathds{1}\!\left\{\hat{u}(x_i) \le \hat{\tau}\right\}}
{\sum\limits_{i=1}^{N}\mathds{1}\!\left\{E_{i}=0\right\}} \\
&\;=\;
\frac{1}{N^P_0}\sum_{i:\,E_{i}=0}\mathds{1}\!\left\{\hat{u}(x_i) \le \hat{\tau}\right\}\ ,
\label{eq:empirical_prompt_coverage_v1}
\end{align}
where $N^P_0=|\{i \in \mathcal{D}_{\mathrm{test}}:\,E_{i}=0\}|$ as the size of set of prompts in the test set whose returned responses are correct. It empirically verifies whether our method meets the theoretical guarantee of covering the accepted prompts with correct returned responses with a probability of at least $1-\alpha$.
Then, the average prediction set size (APS) is computed as,
\begin{equation}
\mathrm{APS}=\frac{1}{N}\sum_{i=1}^N |\mathcal{C}_\alpha(x_i)|\ .
\label{}
\end{equation}
To compute SSCV (Size-Stratified Coverage Violation), let $|\mathcal{C}_\alpha(x_i)|$ be the prediction set size for each prompt and partition the test prompts into $B$ disjoint strata $\{\mathcal{G}_b\}_{b=1}^B$ by fixed ranges of their prediction set size. We report SSCV as the worst-case shortfall across coverage strata $\mathrm{Cov}_b ={1}/{|\mathcal{G}_b|}\sum_{i\in\mathcal{G}_b}\mathds{1}{\{y_i\in\mathcal{C}_\alpha(x_i)\}}$ as,
\[
\mathrm{SSCV}\;=\;\max_{b\in[B]}\Big( (1-\alpha)-\mathrm{Cov}_b \Big)_+\ ,
\]
where $(\cdot)_+=\max\{\cdot,0\}$ and $\mathrm{Cov}_b$ is the stratum-wise empirical coverage.


\section{Additional Experimental Results}
\label{additional_res}

\paragraph{Datasets.}
We additionally evaluate on SamSum~\cite{gliwa2019samsum} to include human-written dialogues paired with abstractive summaries, which enables evaluation of conversational summarization quality.

\paragraph{LLM Models.}
In addition to the 7B-parameter models, we evaluate three 8B-parameter models: Llama-3-8B~\cite{kassianik2025llama}, Qwen-3-8B~\cite{zheng2026empirical}, and Granite-3.1-8B~\cite{granite2024granite}.

\paragraph{Baselines.}
We further compare ACSE against SAR~\cite{duan2024shifting}, KLE~\cite{nikitin2024kernel}, and SIMBA~\cite{bhattacharjya2025simba}, alongside the baselines used in the main experiments.

\paragraph{Discriminative Performance across Models and Tasks.} 
To strengthen timeliness and generalizability, we extend our evaluation beyond QA tasks to include the summarization benchmark SamSum where ACSE consistently outperforms all baselines on the corresponding discriminative metrics. We further conduct same-protocol experiments using more recent models including Llama-3-8B, Qwen-3-8B, and Granite-3.1-8B, against additional baselines such as KLE, SAR, and SIMBA. SAR estimates uncertainty through relevance-weighted semantic consistency across generated responses. On the other hand, KLE captures fine-grained uncertainty using kernel-based semantic similarity distributions while SIMBA improves similarity aggregation. ACSE, in contrast, further addresses cluster brittleness and conformalizes the adjusted uncertainty score to provide finite-sample guarantees. As shown in Table~\ref{tab:recent_baselines_llama3}, ACSE remains the strongest method across both datasets. Gains are observed on SamSum, where ACSE improves AUROC to $0.807$ and consistently lowers FPR@$95$ to $0.447$ over all baselines. As shown in Table~\ref{tab:model_generalization_rebuttal}, ACSE remains the strongest method across all three architectures, achieving consistent improvements over KLE on TriviaQA, e.g., AUROC $0.872$ (Llama-3-8B), $0.869$ (Qwen-3-8B), and $0.863$ (Granite-3.1-8B), along with corresponding reductions in FPR@$95$.

\paragraph{Base SE vs. ACSE Performance Analysis.}
This ablation isolates the performance contribution of the adaptive inflation mechanism by benchmarking base Semantic Entropy (SE) against ACSE across three datasets. As shown in Table~\ref{tab:ablation_inflation}, adaptive inflation is the primary driver of improved error detection, yielding significant absolute AUROC increases up to $+0.12$ on CoQA, and a substantial reduction in FPR@$95$ across all benchmarks. Most notably, the inflation mechanism identifies overconfident errors missed by uninflated signals, reducing the FPR@$95$ by $0.18$ on TriviaQA and CoQA. This gap confirms that raw entropy scores often fails to detect unstable semantic clusters; by leveraging semantic features to inflate uncertainty values for brittle responses, ACSE successfully prevents incorrect answers from passing the conformal threshold.

\paragraph{Probabilistic Model Calibration.} This evaluation assesses how well uncertainty scores reflect the actual likelihood of correctness using Expected Calibration Error (ECE)~\cite{guo2017calibration} and Brier Score. As shown in Table~\ref{tab:probabilistic_calibration}, ACSE significantly outperforms all baselines across all miscoverage levels. For example, ACSE reduces ECE to $0.03$ in the permissive region, a nearly fivefold improvement the overconfident SU baseline. This performance extends to Brier scores, which ACSE consistently maintains below $0.10$, confirming that semantic-level inflation provides a robust confidence signal regardless of the specified tolerance.

\paragraph{Statistical Significance Analysis.}
To complement the mean$\pm$std results, we additionally perform paired significance tests on the main baseline comparison. Since all methods are evaluated on the same prompts, we use paired bootstrap tests over prompts for the primary metrics: AUROC, FPR@95, and AUARC. This allows us to verify that the gains of ACSE over the strongest recent baselines are not explained by run-to-run variation alone.
Paired statistical significance results for Table~\ref{tab:recent_baselines_llama3} are shown in Table~\ref{tab:recent_baselines_significance}. Using a paired bootstrap test over prompts, the gains of ACSE over the strongest baselines remain statistically significant on the main metrics. On TriviaQA, improvements over both KLE and SIMBA are significant across AUROC, FPR@95, and AUARC, with all $p<0.01$. On SamSum, the gains are smaller but remain significant, with all $p<0.02$. These results strengthen the conclusion that the improvements of ACSE are not explained by run-to-run variation alone.


\begin{table}[t]
    \centering
    \resizebox{\columnwidth}{!}{%
    \begin{tabular}{l|c|c|l}
        \toprule
        \textbf{Region} & \textbf{Threshold ($\epsilon$)} & \textbf{AUROC} & \textbf{Impact Explanation} \\
        \midrule
        \multirow{2}{*}{Over-Fragmentation} & 0.10 & 0.72 & Splits paraphrases (High False Positives) \\
         & 0.20 & 0.78 & Fragments minor lexical variations \\
        \midrule
        \textbf{Optimal} & \textbf{0.35} & \textbf{0.83} & \textbf{Correctly groups semantic equivalents} \\
        \midrule
        \multirow{2}{*}{Under-Fragmentation} & 0.50 & 0.75 & Merges distinct but close meanings \\
         & 0.70 & 0.65 & Merges contradictions (Masks Uncertainty) \\
        \bottomrule
    \end{tabular}
    }
	    \caption{Sensitivity Analysis of the clustering threshold $\epsilon$.}
    \label{tab:sensitivity_threshold}
\end{table}

\begin{table}[t]
\centering
\resizebox{\columnwidth}{!}{%
\begin{tabular}{l|cc|cc}
\toprule
\multirow{2}{*}{\textbf{Configuration}} & \multicolumn{2}{c|}{\textbf{TriviaQA}} & \multicolumn{2}{c}{\textbf{CoQA}} \\
 & AUROC & FPR@95 & AUROC & FPR@95 \\
\midrule
\textbf{Full ACSE Framework} & \textbf{0.88} & \textbf{0.31} & \textbf{0.87} & \textbf{0.37} \\
\midrule
w/o Cluster Size Penalty & 0.87 & 0.32 & 0.86 & 0.38 \\
w/o Margin-to-threshold & 0.87 & 0.33 & 0.86 & 0.39 \\
w/o Cluster Dispersion & 0.86 & 0.34 & 0.85 & 0.40 \\
w/o Centroid Distance & 0.85 & 0.35 & 0.84 & 0.42 \\
w/o Cluster-membership Entropy & 0.82 & 0.438  & 0.81 & 0.45 \\
\bottomrule
\end{tabular}
}
\caption{Cluster robustness feature ablation analysis.}
\label{tab:ablation_features}
\end{table}

\begin{table}[t]
\centering
\resizebox{0.8\columnwidth}{!}{%
\begin{tabular}{c|cc}
\toprule
No. of Response Samples & AUROC & AUARC \\ 
\midrule
$n = 4$  & 0.805 & 0.802 \\ 
$n = 7$  & 0.862 & 0.837 \\ 
\midrule
$\mathbf{n = 10}$ & $\mathbf{0.883}$ & $\mathbf{0.851}$ \\ 
\midrule
$n = 13$ & 0.886 & 0.851 \\ 
$n = 16$ & 0.887 & 0.851 \\ 
\bottomrule
\end{tabular}
}
\caption{Sample Size versus ACSE Discriminative Performance.}
\label{tab:sample_size_sensitivity}
\end{table}

\begin{figure}[t]
    \centering
    \begin{subfigure}[b]{0.49\columnwidth}
        \centering
        \includegraphics[width=\linewidth]{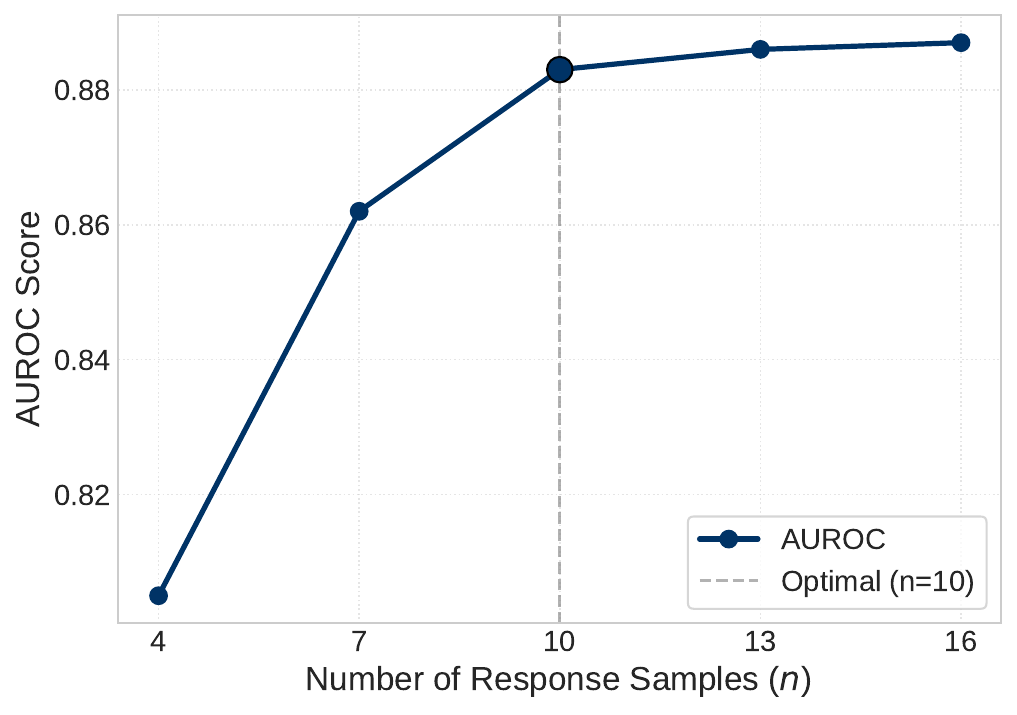}
        \caption{AUROC vs. Sample Size}
    \end{subfigure}
    \hfill
    \begin{subfigure}[b]{0.49\columnwidth}
        \centering
        \includegraphics[width=\linewidth]{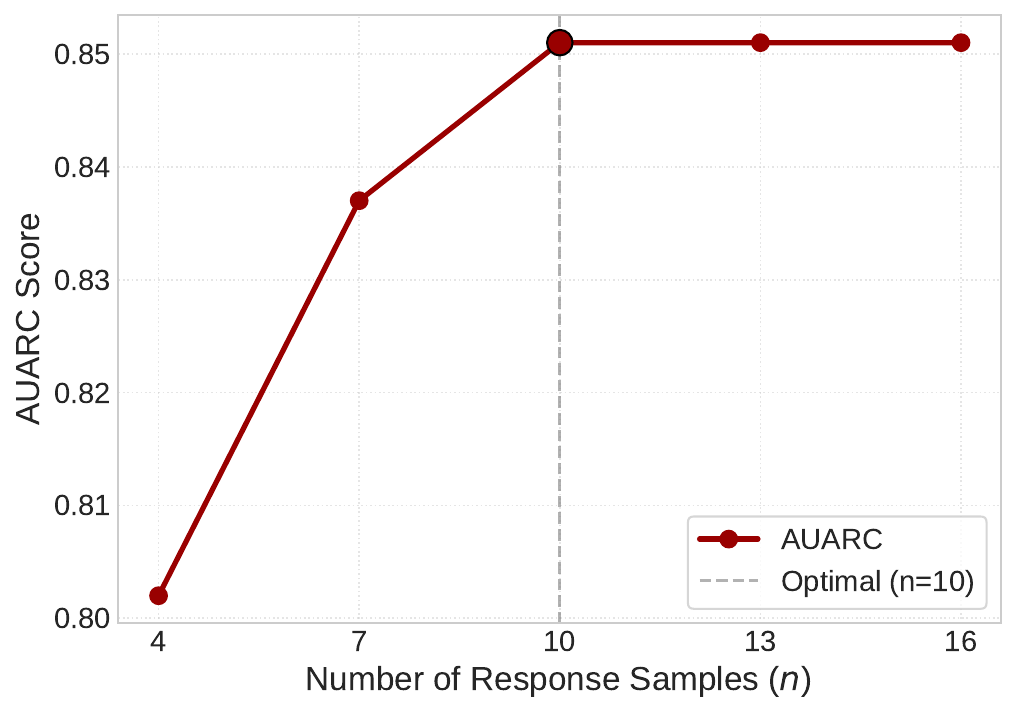}
        \caption{AUARC vs. Sample Size}
    \end{subfigure}
    \caption{Sensitivity analysis of the response sample size $n$.}
    \label{fig:sample_size_sensitivity}
\end{figure}

\begin{figure}[t]
    \centering
    \begin{subfigure}[b]{0.49\columnwidth}
        \centering
        \includegraphics[width=\linewidth]{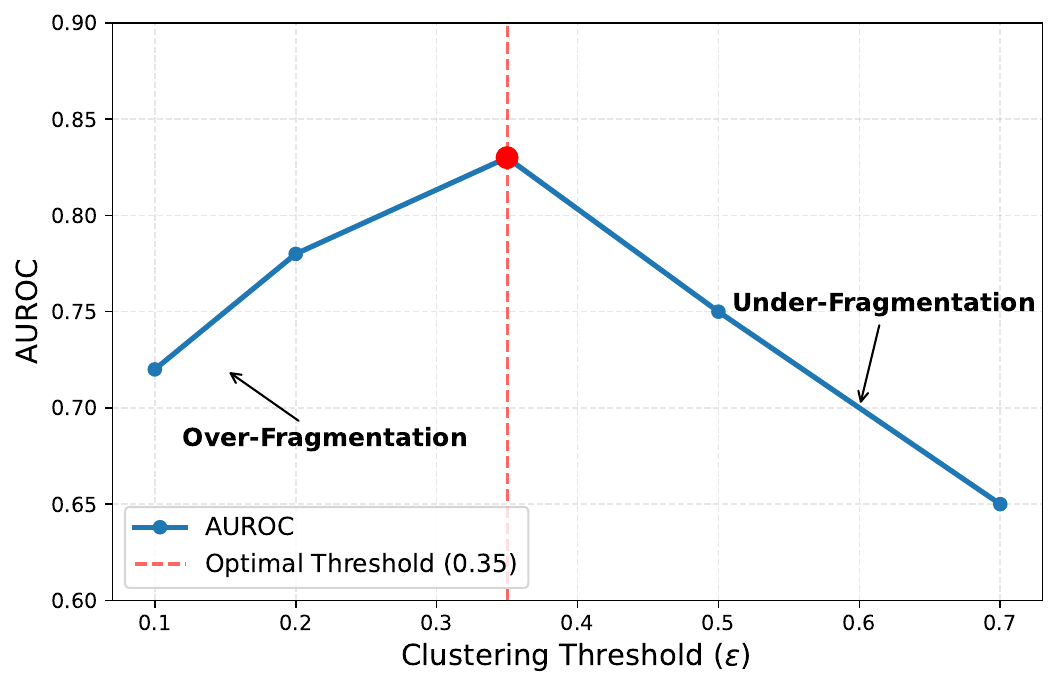}
        \caption{Clustering threshold $\epsilon$}
        \label{fig:threshold_sensitivity}
    \end{subfigure}
    \hfill
    \begin{subfigure}[b]{0.49\columnwidth}
        \centering
        \includegraphics[width=\linewidth]{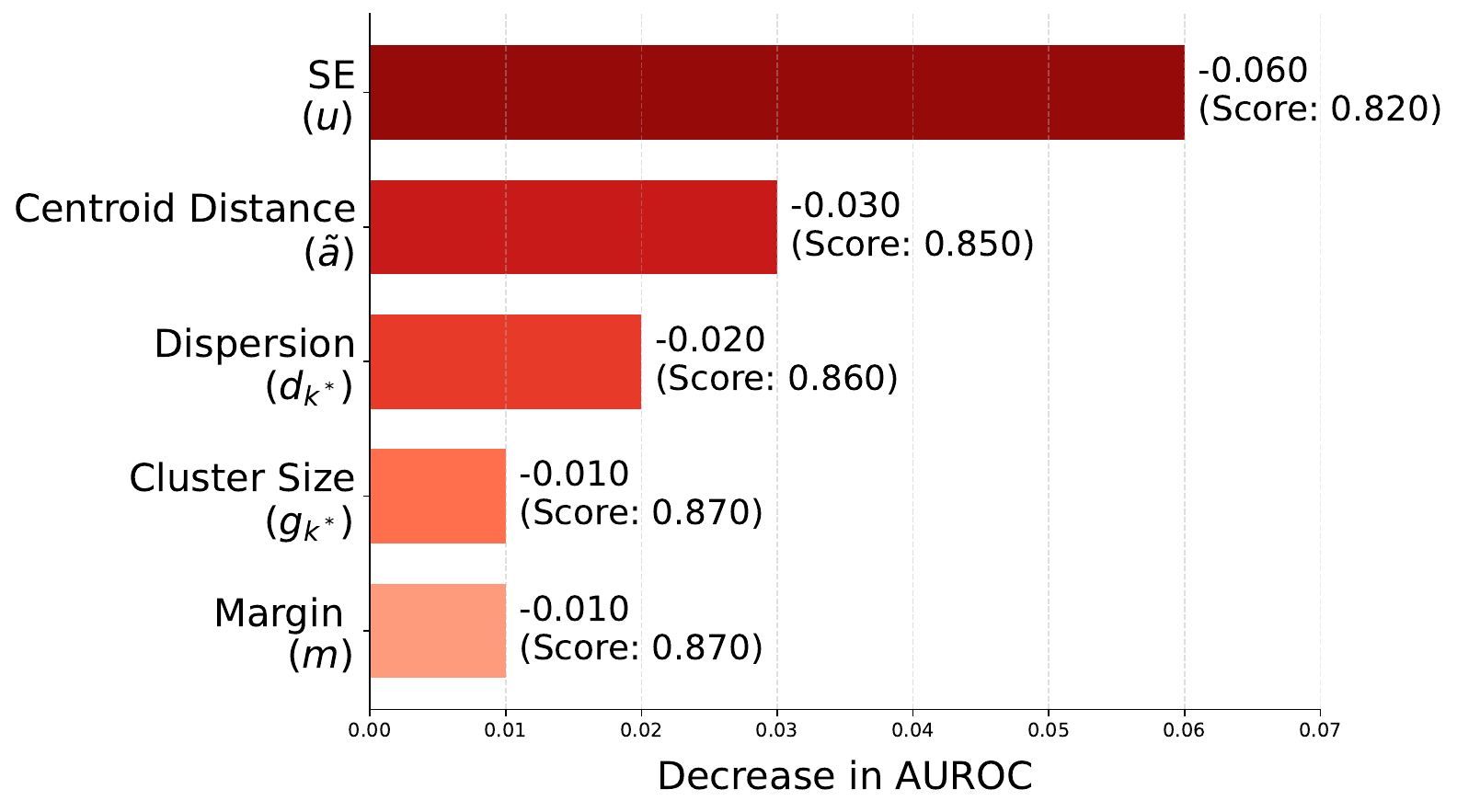}
        \caption{Cluster feature importance}
        \label{fig:ablation_features}
    \end{subfigure}
    \caption{Sensitivity analysis on clustering threshold $\epsilon$ and cluster robustness features (via a leave-one-out ablation study).}
    \label{}
\end{figure}

\begin{table*}[t]
    \centering
    \renewcommand{\arraystretch}{1.10}
    \resizebox{\textwidth}{!}{%
    \begin{tabular}{c l
        *{4}{c}
        *{4}{c}
        *{4}{c}
    }
        \toprule
        \multirow{2}{*}{}
        & \multirow{2}{*}{\textbf{Metric}}
        & \multicolumn{4}{c}{\textbf{DDCRP-CP}}
        & \multicolumn{4}{c}{\textbf{CAP}}
        & \multicolumn{4}{c}{\textbf{ACSE (Ours)}} \\
        \cmidrule(lr){3-6} \cmidrule(lr){7-10} \cmidrule(lr){11-14}
        &
        & \textbf{10\%} & \textbf{20\%} & \textbf{30\%} & \textbf{40\%}
        & \textbf{10\%} & \textbf{20\%} & \textbf{30\%} & \textbf{40\%}
        & \textbf{10\%} & \textbf{20\%} & \textbf{30\%} & \textbf{40\%} \\
        \midrule

        \multirow{3}{*}{\rotatebox[origin=c]{90}{\textbf{Discr.}}}
        & AUROC
        & $.788\pm.007$ & $.794\pm.007$ & $.800\pm.006$ & $.805\pm.006$
        & $.781\pm.008$ & $.788\pm.007$ & $.794\pm.007$ & $.799\pm.006$
        & $\mathbf{.850\pm.005}$ & $\mathbf{.857\pm.005}$ & $\mathbf{.862\pm.005}$ & $\mathbf{.866\pm.004}$ \\
        & FPR@95
        & $.482\pm.013$ & $.473\pm.012$ & $.466\pm.011$ & $.460\pm.011$
        & $.497\pm.014$ & $.486\pm.013$ & $.476\pm.012$ & $.469\pm.012$
        & $\mathbf{.366\pm.010}$ & $\mathbf{.357\pm.010}$ & $\mathbf{.351\pm.009}$ & $\mathbf{.346\pm.009}$ \\
        & AUARC
        & $.728\pm.008$ & $.734\pm.008$ & $.739\pm.007$ & $.743\pm.007$
        & $.722\pm.008$ & $.729\pm.008$ & $.734\pm.007$ & $.739\pm.007$
        & $\mathbf{.824\pm.006}$ & $\mathbf{.828\pm.006}$ & $\mathbf{.831\pm.006}$ & $\mathbf{.833\pm.006}$ \\
        \midrule

        \multirow{6}{*}{\rotatebox[origin=c]{90}{\textbf{Conformal}}}
        & R-Cov
        & $.887\pm.006$ & $.895\pm.006$ & $.902\pm.005$ & $.908\pm.005$
        & $.878\pm.006$ & $.885\pm.006$ & $.891\pm.006$ & $.896\pm.005$
        & $\mathbf{.904\pm.005}$ & $\mathbf{.912\pm.005}$ & $\mathbf{.918\pm.005}$ & $\mathbf{.923\pm.004}$ \\
        & APS
        & $1.36\pm.03$ & $1.33\pm.03$ & $1.30\pm.02$ & $1.28\pm.02$
        & $1.25\pm.03$ & $1.23\pm.03$ & $1.21\pm.02$ & $1.20\pm.02$
        & $\mathbf{1.14\pm.02}$ & $\mathbf{1.13\pm.02}$ & $\mathbf{1.12\pm.02}$ & $\mathbf{1.11\pm.02}$ \\
        & SSCV
        & $.066\pm.004$ & $.062\pm.004$ & $.059\pm.003$ & $.056\pm.003$
        & $.056\pm.004$ & $.053\pm.003$ & $.050\pm.003$ & $.048\pm.003$
        & $\mathbf{.041\pm.003}$ & $\mathbf{.038\pm.003}$ & $\mathbf{.036\pm.003}$ & $\mathbf{.034\pm.002}$ \\
        & P-Cov
        & $.864\pm.006$ & $.871\pm.006$ & $.877\pm.005$ & $.882\pm.005$
        & $.882\pm.006$ & $.891\pm.006$ & $.898\pm.005$ & $.902\pm.005$
        & $\mathbf{.896\pm.005}$ & $\mathbf{.903\pm.005}$ & $\mathbf{.908\pm.005}$ & $\mathbf{.912\pm.004}$ \\
        & Acc.
        & $66.8\pm.7\%$ & $68.4\pm.7\%$ & $69.9\pm.6\%$ & $71.0\pm.6\%$
        & $69.2\pm.7\%$ & $7.6\pm.6\%$ & $71.9\pm.6\%$ & $72.8\pm.6\%$
        & $\mathbf{71.6\pm.6\%}$ & $\mathbf{72.8\pm.6\%}$ & $\mathbf{73.7\pm.5\%}$ & $\mathbf{74.4\pm.5\%}$ \\
        & Risk
        & $.108\pm.005$ & $.102\pm.005$ & $.096\pm.004$ & $.091\pm.004$
        & $.118\pm.005$ & $.111\pm.005$ & $.105\pm.004$ & $.101\pm.004$
        & $\mathbf{.101\pm.004}$ & $\mathbf{.097\pm.004}$ & $\mathbf{.094\pm.004}$ & $\mathbf{.092\pm.004}$ \\
        \bottomrule
    \end{tabular}%
    }
	    \caption{Calibration size sensitivity (Llama-3-8B, TriviaQA).}
    \label{tab:calibration_pool_merged_llama3}
\end{table*}

\begin{table}[t]
    \centering
    \resizebox{0.5\textwidth}{!}{
    \begin{tabular}{l|ccccc}
        \toprule
        \textbf{Weight Setting}
        & {AUROC}
        & {FPR@95}
        & {FPR@90}
        & {AUPR}
        & {AUARC} \\
        \midrule
        Uniform & $.857\pm.005$ & $.357\pm.010$ & $.307\pm.009$ & $.835\pm.005$ & $.828\pm.005$ \\
        Entropy & $.861\pm.005$ & $.351\pm.009$ & $.302\pm.009$ & $.838\pm.005$ & $.831\pm.005$ \\
        Geometry & $.859\pm.005$ & $.353\pm.009$ & $.304\pm.009$ & $.837\pm.005$ & $.830\pm.005$ \\
        Support & $.853\pm.006$ & $.362\pm.010$ & $.312\pm.009$ & $.832\pm.006$ & $.825\pm.006$ \\
        Margin  & $.855\pm.005$ & $.359\pm.010$ & $.309\pm.009$ & $.834\pm.005$ & $.827\pm.005$ \\
        \bottomrule
    \end{tabular}
    }
	    \caption{Discrimination sensitivity over brittleness weighting (Llama-3-8B, TriviaQA).}
    \label{tab:weight_ablation_discrimination_llama}
\end{table}

\begin{table}[t]
    \centering
    \resizebox{0.5\textwidth}{!}{
    \begin{tabular}{l|cccccc|cccccc}
        \toprule
        \multirow{2}{*}{\textbf{Weight Setting}} 
        & \multicolumn{6}{c|}{\textbf{TriviaQA}} 
        & \multicolumn{6}{c}{\textbf{SamSum}} \\
        & {R-Cov} & {APS} & {SSCV} & {P-Cov} & {Acc.} & {Risk}
        & {R-Cov} & {APS} & {SSCV} & {P-Cov} & {Acc.} & {Risk} \\
        \midrule
        Uniform  & 0.920 & 1.10 & 0.035 & 0.911 & 73.9\% & 0.093 & 0.909 & 1.14 & 0.040 & 0.904 & 72.2\% & 0.095 \\
        Entropy  & 0.922 & 1.11 & 0.034 & 0.913 & 73.3\% & 0.091 & 0.910 & 1.15 & 0.039 & 0.905 & 72.0\% & 0.093 \\
        Geometry & 0.921 & 1.09 & 0.033 & 0.912 & 73.7\% & 0.092 & 0.912 & 1.13 & 0.038 & 0.906 & 72.5\% & 0.094 \\
        Support  & 0.917 & 1.12 & 0.037 & 0.908 & 72.9\% & 0.095 & 0.906 & 1.16 & 0.042 & 0.901 & 71.8\% & 0.097 \\
        Margin   & 0.918 & 1.11 & 0.036 & 0.909 & 73.1\% & 0.094 & 0.907 & 1.15 & 0.041 & 0.902 & 71.9\% & 0.096 \\
        \bottomrule
    \end{tabular}
    }
	    \caption{Conformal metrics sensitivity over brittleness weighting at $\alpha=0.10$.}
    \label{tab:weight_ablation_conformal}
\end{table}

\begin{table}[t]
    \centering
    \resizebox{0.5\textwidth}{!}{
    \begin{tabular}{l|ccccc}
        \toprule
        \textbf{Encoder}
        & {AUROC}
        & {FPR@95}
        & {FPR@90}
        & {AUPR}
        & {AUARC} \\
        \midrule
        All-MiniLM & $0.857\pm0.005$ & $0.357\pm0.010$ & $0.307\pm0.009$ & $0.835\pm0.005$ & $0.828\pm0.005$ \\
        MPNet  & $0.860\pm0.005$ & $0.352\pm0.009$ & $0.303\pm0.009$ & $0.838\pm0.005$ & $0.831\pm0.005$ \\
        E5     & $0.855\pm0.005$ & $0.360\pm0.010$ & $0.310\pm0.009$ & $0.833\pm0.005$ & $0.826\pm0.005$ \\
        BGE    & $0.858\pm0.005$ & $0.355\pm0.010$ & $0.306\pm0.009$ & $0.836\pm0.005$ & $0.829\pm0.005$ \\
        \bottomrule
    \end{tabular}
    }
	    \caption{Encoder Choice Sensitivity (Llama-3-8B, TriviaQA).}
    \label{tab:encoder_discrimination_llama}
\end{table}

\begin{table}[!t]
    \centering
    \resizebox{0.5\textwidth}{!}{
    \begin{tabular}{l|cccccc|cccccc}
        \toprule
        \multirow{2}{*}{\textbf{Encoder}} 
        & \multicolumn{6}{c|}{\textbf{TriviaQA}} 
        & \multicolumn{6}{c}{\textbf{SamSum}} \\
        & {R-Cov} & {APS} & {SSCV} & {P-Cov} & {Acc.} & {Risk}
        & {R-Cov} & {APS} & {SSCV} & {P-Cov} & {Acc.} & {Risk} \\
        \midrule
        MiniLM & 0.920 & 1.10 & 0.035 & 0.911 & 73.9\% & 0.093 & 0.909 & 1.14 & 0.040 & 0.904 & 72.2\% & 0.095 \\
        MPNet  & 0.922 & 1.09 & 0.034 & 0.913 & 74.2\% & 0.091 & 0.912 & 1.12 & 0.038 & 0.907 & 72.6\% & 0.093 \\
        E5     & 0.918 & 1.11 & 0.036 & 0.910 & 73.7\% & 0.094 & 0.907 & 1.15 & 0.041 & 0.903 & 72.0\% & 0.096 \\
        BGE    & 0.921 & 1.09 & 0.034 & 0.912 & 74.1\% & 0.092 & 0.911 & 1.12 & 0.038 & 0.906 & 72.5\% & 0.093 \\
        \bottomrule
    \end{tabular}
    }
	\caption{Encoder Choice Sensitivity Analysis at $\alpha=0.10$.}
    \label{tab:encoder_conformal}
\end{table}


\subsection{Ablation Studies}
\label{ablation}
We conduct comprehensive ablation studies to examine ACSE's performance, covering brittleness feature contributions, clustering threshold and sampling size sensitivities, calibration size behavior, and robustness to implementation choices such as weighting schemes and sentence encoders.

\paragraph{Sensitivity to Clustering Threshold.}
This experiment evaluates the robustness of ACSE to the semantic clustering threshold $\epsilon$, which controls cluster resolution by setting the minimum similarity required to group responses. By varying $\epsilon$ from $0.10$ to $0.70$, we investigated the trade-off between over-fragmentation and under-fragmentation. As detailed in Table~\ref{tab:sensitivity_threshold} and Figure~\ref{fig:threshold_sensitivity}, the discriminative performance follows a clear bell-shaped curve that peaks at an optimal threshold of $\epsilon=0.35$ with an AUROC of $0.83$.  While low thresholds ($\epsilon \le 0.20$) are overly restrictive and high thresholds ($\epsilon \ge 0.50$) mask uncertainty by merging distinct meanings, the peak at 0.35 successfully captures linguistic invariances for robust uncertainty estimation.

\paragraph{Contribution of Cluster Brittleness Features.}
This study uses a leave-one-out ablation to quantify the contribution of each semantic feature within the inflation function $\lambda$ towards discriminative performance. 
As shown in Table~\ref{tab:ablation_features} and Figure~\ref{fig:ablation_features}, Cluster-membership Entropy is the most critical feature; its removal results in $0.06$ AUROC drop, confirming that mapping ambiguity is the strongest predictor of factual errors. Centroid Distance and Cluster Dispersion follow as secondary signals, indicating that the internal tightness of a cluster provides essential confidence information that raw entropy lacks. Conversely, the minimal impact of removing the Cluster Size Penalty suggests that the structural arrangement and stability of samples are more informative than their raw count. 
The results show the marginal effect of removing each feature from the full ACSE system that measures each feature's contribution in the presence of the others, not its standalone importance or optimal weight, e.g., the larger drop from removing SE indicates it is the strongest signal, not that other features should be down-weighted. 

\paragraph{Sample Size Sensitivity Analysis.}
This experiment evaluates the impact of response sample size ($n$) on the framework's ability to distinguish between hallucinations and correct responses. By varying $n$ from $4$ to $16$, we aimed to identify the elbow point where the increase in computational cost no longer yields significant performance gains. As shown in Table~\ref{tab:sample_size_sensitivity} and Figure~\ref{fig:sample_size_sensitivity}, discriminative performance significantly improves as the sample size increases from $n=4$ to $n=10$.  However, beyond $n=10$, both metrics plateau; AUARC remains constant at $0.851$, and increases in AUROC become marginal. This confirms that $10$ samples provide the optimal trade-off between semantic resolution and inference cost.

\paragraph{Calibration Size Sensitivity.} 
The calibration set should satisfy exchangeability assumption (Section~\ref{conformal_cse}) which is less restrictive than being from the same distribution. To measure sensitivity to the calibration set size, we include an ablation study in Table~\ref{tab:calibration_pool_merged_llama3} varying the calibration fraction from 10\% to 40\%. ACSE consistently provides the strongest overall tradeoff between coverage, compact prediction sets, acceptance, and selective risk across all calibration sizes. At $10\%$ calibration, ACSE achieves higher prompt coverage than CAP and DDCRP-CP ($0.904$ vs.\ $0.889$ and $0.872$), together with the highest acceptance rate ($72.8\%$ vs.\ $70.2\%$ and $68.1\%$). A similar trend is observed  for discrimination and safety metrics. ACSE consistently outperforms both CAP and DDCRP-CP in AUROC, FPR@95, and AUARC across all calibration fractions. At $10\%$ calibration, ACSE improves AUROC to $0.861$, compared with $0.792$ for CAP while reducing FPR@95 to $0.352$ from $0.486$ respectively. The same ordering remains at $20\%$--$40\%$, with only gradual gains as the calibration pool increases. These results indicate that ACSE's advantage is not an artifact of extensive calibration, but remains visible even under substantially smaller calibration budgets.

\paragraph{Weight Sensitivity of Brittleness Features.} 
In the experiments, we use \(\mathbf w=\langle\frac15,\frac15,\frac15,\frac15,\frac15\rangle\) as a deliberate design choice (not as learned parameters fit to labels). Since all five brittleness components are normalized to $[0,1]$, equal weights provide a fair and neutral choice without privileging any feature a priori. This does not imply that equal weighting is required; different deployments may reasonably emphasize different semantic risks (e.g., ambiguity vs. overconfidence vs. sparse support) by user-controlled prioritization over features. Thus, we additionally conduct an ablation on the weight sensitivity in five settings: \textit{Uniform}~$\langle\frac15,\frac15,\frac15,\frac15,\frac15\rangle$, \textit{Entropy}~$\langle\frac26,\frac16,\frac16,\frac16,\frac16\rangle$, \textit{Geometry}~$\langle\frac17,\frac27,\frac27,\frac17,\frac17\rangle$, \textit{Support}~$\langle\frac16,\frac16,\frac16,\frac26,\frac16\rangle$, and \textit{Margin}~$\langle\frac16,\frac16,\frac16,\frac16,\frac26\rangle$, to evaluate robustness under different reasonable weight choices.

As reported in Table~\ref{tab:weight_ablation_discrimination_llama}, ACSE remains stable under all tested weight settings on both TriviaQA and SamSum. On SamSum AUROC ranges from $0.796$ to $0.803$ and FPR@$95$ from $0.448$ to $0.461$. Entropy and Geometry provide slight improvements on some metrics, but the differences are small overall. This suggests that ACSE’s discrimination and safety gains do not depend on a fragile or narrowly tuned weighting choice. A similar pattern holds for the conformal metrics. As also reported in Table~\ref{tab:weight_ablation_conformal}, on TriviaQA prompt coverage stays within $0.908$ to $0.913$ and selective risk within $0.091$ to $0.095$ across all settings, while on SamSum prompt coverage ranges from $0.901$ to $0.906$ and risk from $0.093$ to $0.097$. APS and SSCV also change only marginally. Although Entropy and Geometry are slightly favorable on a few metrics, Uniform remains consistently competitive. These results indicate that ACSE is robust to moderate changes in the brittleness weight vector, and its gains do not rely on fragile hyperparameter tuning.

\paragraph{Sentence Encoder Sensitivity.} 
Encoder quality can affect discriminative performance, but not conformal validity which only requires the same deterministic scoring pipeline in calibration and inference. We conduct an ablation on different sentence encoders based on discriminative performance and conformal analysis. ACSE relies only on a semantically meaningful similarity geometry, while soft assignments reduce brittleness near cluster boundaries. 
As reported in Table~\ref{tab:encoder_discrimination_llama}, ACSE remains stable across all tested sentence encoders on both TriviaQA and SamSum. MPNet achieves the strongest overall discrimination performance, with AUROC of  $0.803$ on SamSum, but the gains over the MiniLM default are small, amounting to only $0.004$, respectively. BGE remains very close to MPNet, while E$5$ is slightly weaker across most metrics. A similar pattern is observed in Table~\ref{tab:encoder_conformal} for the conformal metrics. MPNet and BGE provide slightly stronger conformal performance, but MiniLM remains highly competitive, with only small absolute differences such as about $0.002$ in P-Cov and $0.002$ in Risk relative to the strongest encoder settings. These results indicate that ACSE is robust to moderate changes in the sentence encoder, and that its gains do not depend on a narrowly chosen embedding backbone.


\end{document}